\documentclass[twoside,11pt]{article}
\usepackage{jair, theapa, rawfonts,comment,amsmath}

\jairheading{80}{2024}{12-15}{02/2024}{08/2024}
\ShortHeadings{Uncertainty as a Fairness Measure}
{Kuzucu, Cheong, Gunes \& Kalkan}
\firstpageno{1}

\usepackage{graphicx}
\usepackage{subfigure}
\usepackage{array,booktabs,arydshln,xcolor}
\newcommand{\Unc}{\mathcal{U}}
\newcommand{\Uncp}{\mathcal{U}_p}
\newcommand{\Unca}{\mathcal{U}_a}
\newcommand{\Unce}{\mathcal{U}_e}
\newcommand{\PMeasure}{\mathcal{M}}
\newcommand{\FMeasure}{\mathcal{F}}
\newcommand{\highlight}[1]{\textbf{\textcolor{red}{#1}}}

\usepackage{booktabs}
\usepackage{lipsum}
\usepackage{bbm}
\usepackage{arydshln}
\usepackage{amsmath}
\usepackage{amsthm}
\usepackage{amssymb}
\usepackage{mathtools}
\usepackage{comment}
\newtheorem{definition}{Definition} [section]
\newtheorem{proposition}{Proposition}[section]

\DeclareMathOperator*{\argmax}{arg\,max}

\usepackage{amsmath,amssymb,amsfonts}
\usepackage{algorithmic}
\usepackage{graphicx}
\usepackage{textcomp}
\usepackage{xcolor}
\usepackage{comment}
\usepackage{centernot} 

\begin{document}

\title{Uncertainty as a Fairness Measure}

\author{\name Selim Kuzucu 
\email selim.kuzucu@metu.edu.tr \\
       \addr Department of Computer Engineering\\ Middle East Technical University \\ 
       06800 Ankara, Turkiye
       \AND
       \name Jiaee Cheong 
       \email jc2208@cam.ac.uk \\
       \addr 
       Department of Computer Science \\
       University of Cambridge \\
       Cambridge, CB3 0FD, United Kingdom \\
       The Alan Turing Institute \\
       London, NW1 2DB, United Kingdom
       \AND
        \name Hatice Gunes \email hg410@cam.ac.uk\\
         \addr 
         Department of Computer Science \\
         University of Cambridge\\ 
         Cambridge, CB3 0FD, United Kingdom
        \AND
       \name Sinan Kalkan \email skalkan@metu.edu.tr \\
       \addr Department of Computer Engineering \& \\
       ROMER Robotics-AI Center\\ 
       Middle East Technical University \\
       06800 Ankara, Turkiye
}

\maketitle

\begin{abstract}
Unfair predictions of machine learning (ML) models impede their broad acceptance in real-world settings. Tackling this arduous challenge first necessitates defining what it means for an ML model to be fair. This has been addressed by the ML community with various measures of fairness that depend on the prediction outcomes of the ML models, either at the group-level or the individual-level. These fairness measures are limited in that they utilize point predictions, neglecting their variances, or uncertainties, making them susceptible to noise, missingness and shifts in data. In this paper, we first show that a ML model may appear to be fair with existing point-based fairness measures but biased against a demographic group in terms of prediction uncertainties. Then, we introduce new fairness measures based on different types of uncertainties, namely, aleatoric uncertainty and epistemic uncertainty. 
We demonstrate on many datasets that (i) our uncertainty-based measures are complementary to existing measures of fairness, and (ii) they provide more insights about the underlying issues leading to bias. 
%
\end{abstract}


\section{Introduction}

An impedance to the wide-spread use of machine learning (ML) approaches is the bias present in their predictions against certain demographic groups. The severity and extent of this matter have been considerably investigated for different applications, such as gender recognition \cite{buolamwini2018gender}, emotion or expression recognition \cite{domnich2021responsible,xu2020investigating,chen2021understanding} and mental health prediction \cite{cheong2023towards} etc. 

\begin{figure}[hbt!]
    \centering
    \includegraphics[width=0.7\columnwidth]{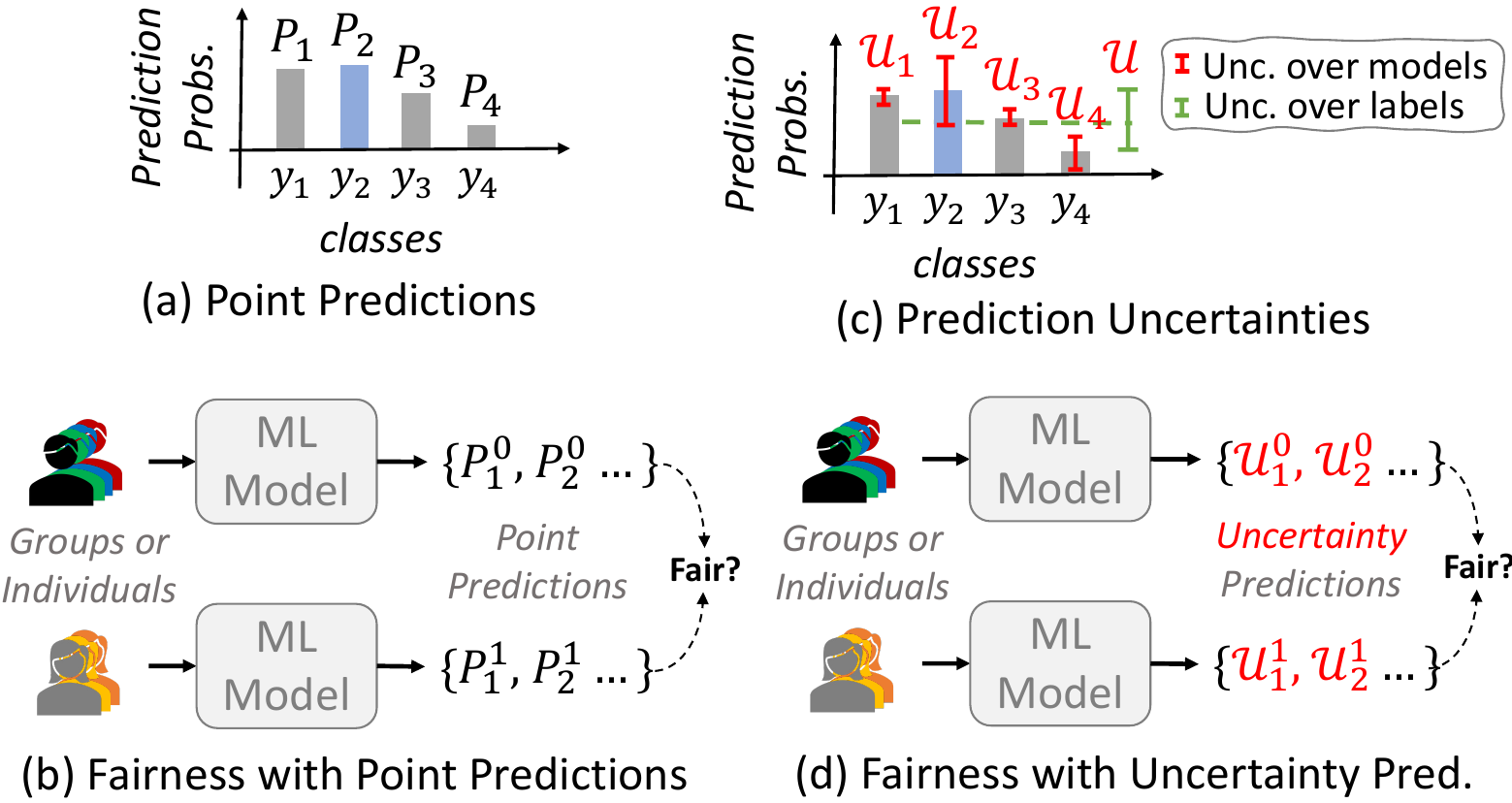}
    \caption{Existing fairness measures utilize point predictions for quantifying fairness, which ignores the uncertainty (variance) of the predictions \textbf{(a-b)}. We fill this gap by using uncertainty instead for measuring fairness \textbf{(c-d)}.}
    \label{fig:teaser}
    \vspace{-0.5cm}
\end{figure}

It has been identified in the literature that fairness is a multi-faceted concept, which has led to different notions and definitions of fairness \cite{garg2020fairness,verma2018fairness,castelnovo2022clarification,dwork2012fairness,mehrabi2021survey}. For example, a model can be considered fair at a group-level (called \textit{group fairness}) if its predictions are the same for the different demographic groups (a.k.a., \textit{Statistical Parity} -- \cite{dwork2012fairness,mehrabi2021survey,garg2020fairness,verma2018fairness}) or if its false negative rates are the same (a.k.a., \textit{Equal Opportunity} -- \cite{hardt2016equality}). 
Alternatively, a model can be evaluated for fairness at the level of individuals (called \textit{individual fairness}) by comparing an individual's predictions to similar individuals \cite{dwork2012fairness} or to a counterfactual version of the individual (called \textit{counterfactual fairness} -- \cite{kusner2017counterfactual,cheong2022counterfactual}).

Despite advances in fairness quantification measures using point predictions (Fig. \ref{fig:teaser}(a,b)) and bias mitigation methods, the utility of such measures or methods are often limited when exposed to real-world data.
This is because (\textbf{P1}) first, they often do not account for real-world problems such as missing data \cite{goel2021importance}, biased labeling \cite{jiang2020identifying} and domain or distribution shifts \cite{chen2022fairness}.
(\textbf{P2}) Second, they are susceptible to fairness gerrymandering.
For instance, depending on how a group is defined, a key challenge with existing statistical-parity point-based fairness measures is that it is implausible to ensure they hold for every subgroup of the population. 
Any classifier can be deemed unfair to the subgroup of individuals defined ex-post as the set of samples it misclassified \cite{kearns2018preventing}. 
(\textbf{P3}) Third, recent works have demonstrated how traditional bias mitigation methods do not necessarily lead to fairer outcomes as measured using traditional parity-based measures nor do they shed light on the source of bias. For instance, larger or more balanced datasets did not mitigate the embedded disparities in \textit{real-world tabular datasets} \cite{ding2021retiring} and balancing samples across gender did not produce fairer predictions for females \cite{cheong2023towards}.
We propose addressing these challenges by measuring fairness using prediction uncertainties (Fig. \ref{fig:teaser}(c,d)). 

\subsection{Uncertainty-based Fairness for Social Impact.} An uncertainty-based definition of fairness has the potential of addressing the aforementioned drawbacks (P1-P3):

\begin{enumerate}
    \item \textbf{Addressing P1}: Point predictions calculated using $P(Y | X)$ are often unreliable \cite{guo2017calibration,baltaci2023class,FocalLoss_Calibration} and uninformative in real-world problems with missing data, labeling or data noise and distribution shifts.  Uncertainty-based fairness addresses {P1} by quantifying the level of unreliability present to provide practitioners with an indication of the potential source of underlying bias present. As we will show in our paper, (i) machine learning models have different prediction uncertainties for different demographic groups and (ii) these differences can provide useful insights, e.g., about a lack of data or a presence of noise affecting one demographic group more than others, which are fundamental to fairness.

    \item \textbf{Addressing P2}: Prediction uncertainties quantify variance over multiple predictions for the same input, which make them less susceptible or less vulnerable towards manipulation. They represent the inherent uncertainty about the model and the data, which is largely immutable for a given model and a dataset.

    \item \textbf{Addressing P3}: Quantification of different types of uncertainty by definition provides insights about the underlying issues with the data and the model, which can shed light on e.g. when adding more data does not necessarily lead to fairer outcomes.
   
\end{enumerate}

\subsection{Contributions.} In summary, our main contributions are:

\begin{itemize}
    \item We \textbf{introduce uncertainty-based fairness measures at the group and individual-level}. To the best of our knowledge, our paper is the first to use uncertainty as a fairness measure. 

    \item We \textbf{prove that an uncertainty-based fairness measure is complementary to point-based measure}s, suggesting that both uncertainty and point predictions should be taken into account when analyzing fairness of models.

    \item We show on many datasets that (i) uncertainty fairness can vary significantly across demographic groups and (ii) it \textbf{provides insight about the sources of bias}.
\end{itemize}

\section{Related Work}

\subsection{Fair ML}

The seminal work of \cite{buolamwini2018gender} and the follow-up studies \cite{domnich2021responsible,xu2020investigating,chen2021understanding,cheong2023causal} have exposed significant bias present in many applications of ML models. To address such biases and obtain fairer ML models, the ML community have proposed a plenitude of pre-processing, in-processing or post-processing strategies with promising outcomes -- see \cite{barocas2017fairness,mehrabi2021survey,cheong2021hitchhiker} for surveys.

\subsection{Fairness Measures}

Fairness has multiple facets, which have been recognized by the ML community with different notions and measures of fairness. One prominent notion of fairness is \textit{group fairness}, which pertains to comparing a model's predictions across different demographic groups. Statistical Parity \cite{dwork2012fairness,mehrabi2021survey,garg2020fairness}, Equal Opportunity \cite{hardt2016equality}, and Equalized Odds \cite{hardt2016equality} are commonly used measures of group fairness.
Alternatively, ML model predictions can be evaluated for \textit{individual fairness} \cite{dwork2012fairness}. Such fairness can be measured e.g. by comparing an individual's predictions with those of similar individuals \cite{dwork2012fairness} or with those of a counterfactual version of the individual \cite{kusner2017counterfactual}.

Wang et al. \cite{wang2023aleatoric} study algorithmic discrimination with two measures of group fairness that are relevant to our fairness measures: Aleatoric discrimination, for inherent biases in data, and epistemic discrimination, for model or algorithmic biases. These discrimination measures are based on the gap between the performance of a model and the fairness Pareto frontier for that model. The fairness Pareto frontier represents the best achievable performance for a certain fairness constraint. The gap between this frontier and the 100\% performance would characterize irreducible (aleatoric) discrimination of the model whereas the gap between the frontier and the current model's performance would represent reducible (epistemic) discrimination. Although these measures are valuable, obtaining the fairness Pareto frontier requires solving a sophisticated optimization problem. Wang et al. address this issue by making simplifications about the decision boundaries or the machine learning model, which limit the applicability of their approach in practice. Moreover, their approach is limited to only measuring group-level discrimination.

{
\subsection{Bayesian Neural Networks (BNNs)}

Bayesian Neural Networks \cite{mackay1992practical,neal1995bayesian} operate by placing a prior distribution over the weights of a neural network, such that each weight is represented by a distribution parameterized by a mean and a standard deviation: $\omega_i = (\mu_i, \sigma_i)$.
In addition to being robust against over-fitting \cite{gal2016uncertainty,blundell2015weight}, BNNs are known to work well with small datasets and propagate reliable uncertainty estimates as they bring a natural framework to estimate the first two moments of the predictive distribution \cite{gal2016uncertainty}.

Against these desirable properties, one key challenge with BNNs is to perform inference.
Inference arise as a challenge due to the need to find the most probable weights (in the form of distributions) that have generated the data.
Specifically, when we attempt to apply the Bayes' Theorem to obtain the true posterior on the weights, we often fail to do so as marginalizing the prior on the weights does not have an analytical solution for the complex cases \cite{gal2016uncertainty}.
Owing to this issue, there is a need for approximating the posterior, which is often achieved by utilizing variational Bayesian approximation techniques such as Monte Carlo (MC) sampling.
In practice, MC sampling is used not only for optimizing the BNNs (Section \ref{sect:training_details}) but also for uncertainty quantification (Section \ref{sect:uncertainty_and_quantification}).
%

\subsection{BNNs and Uncertainty Quantification}
Modern decision making systems should possess an awareness of unknowns and propagate this information to the people at the end of the decision making pipelines.
In recent literature, this goal is aimed to be achieved by producing reliable uncertainty estimates, commonly referred to as uncertainty quantification \cite{gal2016dropout,kendall2017uncertainties,mukhoti2023ddu,van2020duq}.

In Bayesian modeling, it is possible to disentangle the overall predictive uncertainty into two unique components: epistemic uncertainty to account for the lack of data and aleatoric uncertainty to account for any irreducible uncertainty associated with the data \cite{kendall2017uncertainties}.
Epistemic uncertainty is modeled by capturing how much the weights vary given a set of data based on a prior placed on the weights whereas aleatoric uncertainty is modeled by quantifying the variance of the distribution placed over the outputs of the model \cite{kendall2017uncertainties,kwon2020uncertainty}.
Due to the aforementioned intractability issue, both of these uncertainties are commonly captured using MC sampling \cite{gal2016dropout,kendall2017uncertainties,kwon2020uncertainty}.
Recently, by improving the methodology introduced in \cite{kendall2017uncertainties}, 
Kwon \textit{et al.}
\shortcite{kwon2020uncertainty} proposed a novel method to quantify these two components without the need to optimize for a separate variance parameter, which we utilize in our work and formally describe in Section \ref{sect:uncertainty_and_quantification}.

\subsection{Uncertainty and Fairness}

The relationship between the uncertainty and fairness has been subject to numerous recent studies, such as \cite{mehta2023evaluating,tahir2023fairness,kaiser2022uncertainty}.
For example, \cite{mehta2023evaluating} have shown that mitigating bias has an adverse affect on the (predictive) uncertainty of estimations. Furthermore, \cite{tahir2023fairness} with \cite{kaiser2022uncertainty} utilize aleatoric uncertainty as part of their bias mitigation strategy. Our work is also distinct from existing work that attempts to quantify the uncertainty of a fairness measure \cite{roy2023fairness} or the bias present \cite{ethayarajh2020your}. 
%
In addition, it is well-known that ML models are generally under- or over-confident \cite{guo2017calibration,FocalLoss_Calibration} or unreliable under noise \cite{kendall2017uncertainties}. To illustrate that this may also be the case in datasets which are commonly used in fairness analysis, we plot the prediction confidence of the BNN classifier on the COMPAS dataset in Fig. \ref{fig:compas_calibration}. The plot shows that the model is under or over-confident about its predictions and that there is a so-called calibration gap.
\begin{figure}[hbt!]
  \centering
  \includegraphics[width=0.5\columnwidth]{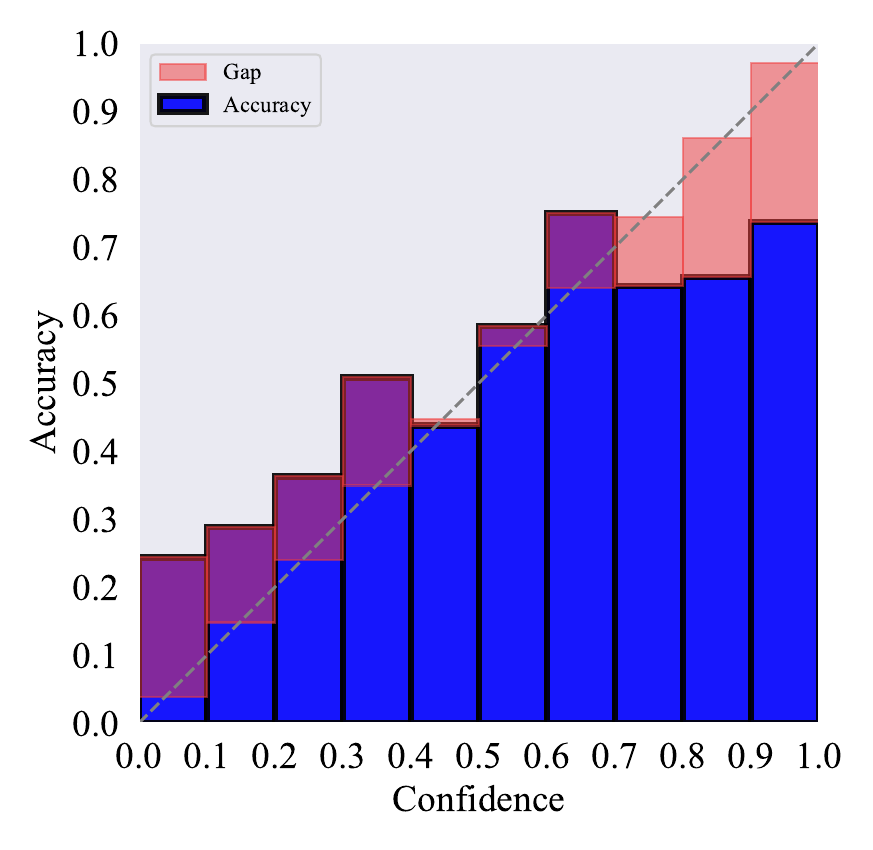}
  \caption{\textbf{Illustration of Under- or Over-Confidence:} An example illustrating under- or over-confidence of an ML model's predictions. The diagram is calculated for a Bayesian NN classifier on the COMPAS Dataset, a dataset frequently used in ML fairness research.}
  \label{fig:compas_calibration}
\end{figure}

\subsection{Comparative Summary}
As discussed, existing fairness measures have only considered point predictions, which provide an incomplete view about the quality of a model's predictions. To the best of our knowledge, we are the first to address this gap by proposing uncertainty as a fairness measure. We prove that the introduced uncertainty measure is \textit{complementary} to the point-based fairness measures. We showcase that (i) point-based fairness and uncertainty-based fairness can be complementary, and (ii) uncertainty fairness can provide insight about the sources of bias on several datasets.

\section{Preliminaries and Background}
In this section, we introduce some preliminary notation and background knowledge to aid subsequent understanding of the paper.

\subsection{Notation}

Following the setting and notation in the literature \cite{verma2018definitions,castelnovo2022clarification}, we assume a binary classification problem with a dataset $D$ of $X$, $Y$ and $G$ where $X$ denotes features describing an individual, $Y \in \{0, 1\}$ is the classification target, and $G \in \{0, 1\}$ is the majority group indicator, with $G=0$ denoting the minority. 
Solving the classification problem involves finding a mapping 
$\hat{Y} = f(X; \theta) \in \{0, 1\}$ with parameters $\theta$. We use $P(Y = y_i | X = \mathbf{x}_i)$ to denote the predicted probability for the correct class $y_i$ for sample $X = \mathbf{x}_i$, and $\hat{Y}=\hat{y}_i \leftarrow \argmax_{c} P(Y = c | X = \mathbf{x}_i)$ 
to denote the predicted class.

\subsection{Measuring Group Fairness}
\label{sect:measuring_fairness}

A ML model can be considered fair if a chosen performance measure for a specific task is the same across different groups \cite{garg2020fairness,verma2018fairness}. More formally, for a predictor $\hat{Y} = f(\ \cdot\ ; \theta)$ to be considered fair with respect to a demographic group attribute $G$, the following equality should be met for a given performance measure $\mathcal{M}$, e.g., true positive rate:
{
\begin{equation}
 \textrm{Fair}(f; \mathcal{M}, D) \equiv \mathcal{M}(D, f, G=0) = \mathcal{M}(D, f, G=1).
\end{equation}
}
Existing work exploring different performance measures for $\mathcal{M}$ has shown that each entails a different notion of fairness. For example: 

\paragraph{Statistical Parity, or Demographic Parity} \cite{dwork2012fairness,mehrabi2021survey,garg2020fairness,verma2018fairness}:  Compares model's prediction probabilities for the positive class ($\hat{Y} = 1$) across different groups (with $\mathcal{M}(D, f, G) \equiv P(\hat{Y}=1 | G)$):
\begin{equation}
            P(\hat{Y}=1 | G=0) = P(\hat{Y}=1 | G=1).
\end{equation}

\paragraph{Equal Opportunity} \cite{hardt2016equality}: Compares model's false negative rates, i.e., prediction probabilities for the negative class ($\hat{Y} = 0$) for the known positive class ($Y = 1$):
{
\begin{equation}
            P(\hat{Y}=0 | Y=1, G=0) = P(\hat{Y}=0 | Y=1, G=1),
\end{equation}
}
where $\mathcal{M}(D, f, G) \equiv P(\hat{Y}=0 | Y = 1, G)$.

\paragraph{Equalised Odds} \cite{hardt2016equality}: Compares model's prediction probabilities for the positive class ($\hat{Y} = 1$) for different ground truth classes ($Y = 1$ and $Y = 0$):
{
\begin{equation}
            P(\hat{Y}=1 | Y=y, G=0) = P(\hat{Y}=1 | Y=y, G=1),
\end{equation}
}
where $y\in \{0, 1\}$, and we've taken $\mathcal{M}(D, f, G) \equiv P(\hat{Y}=1 | Y = y, G)$.

\subsection{Measuring Individual Fairness}

Dwork \textit{et al.}
\shortcite{dwork2012fairness} defined individual fairness based on a ``similar individuals should have similar predictions'' principle:
\begin{equation}\label{eqn:consistency}
    d_y(f(\mathbf{x}_1), f(\mathbf{x}_2)) \le Ld_x(\mathbf{x}_1, \mathbf{x}_2), \quad \forall \mathbf{x}_1, \mathbf{x}_2 \in \mathcal{X}.
\end{equation}
The above notion assumes suitable distance metrics $d_y(\cdot, \cdot)$ and $d_x(\cdot, \cdot)$ to be available for the predictions and the inputs respectively. The literature has used point predictions to quantify this notion of fairness e.g. by using a consistency measure \cite{zemel2013learning,mukherjee2020two}:
\begin{equation}
   \mathcal{F}^{indv}_{\hat{y}}(X=\mathbf{x}_i) = 1-\left| \hat{y}_i - \frac{1}{k} \sum_{\mathbf{x}_j \in k\textrm{NN}(\textbf{x}_i)} \hat{y}_j \right|,
\end{equation}
where $k\textrm{NN}(\textbf{x}_i)$ denotes the $k$-nearest neighbours of $\mathbf{x}_i$.
\section{Methodology}
\label{sect:uncertainty_fairness}

We first describe how we quantify uncertainty and then introduce the fairness measures.
\subsection{Quantifying Uncertainty}
\label{sect:uncertainty_and_quantification}

ML models tend to be under- or over-confident about their predictions and unaware of distribution shift, adversarial attacks or noise in data \cite{abdar2021review,gawlikowski2021survey,cetinkaya2024ranked}. Quantifying the variance of a model's predictions, i.e., \textit{predictive uncertainty}, facilitates awareness of such hindrances with respect to the data. Predictive uncertainty has two components, reflecting the two different ways to define a variance over predictions:

\begin{itemize}
    \item \textbf{Epistemic or model uncertainty} is measured over different models. Epistemic uncertainty reflects the lack of knowledge about the current input and can be reduced by providing more training data, i.e., more knowledge.
    \item \textbf{Aleatoric or data uncertainty} is measured over classes. Aleatoric uncertainty reflects the irreducible noise in the data.
\end{itemize}

We use Bayesian Neural Networks (BNNs) to obtain uncertainty estimates as described in \cite{blundell2015weight} since BNNs provide reliable uncertainty estimations.
A BNN
defines a distribution over each weight in the model: $\omega_i = (\mu_i, \sigma_i)$, which enables sampling different weights and making multiple predictions for the same input.
With such a model, predictive uncertainty for a sample $\mathbf{x}$ with label $y$ can be quantified as follows \cite{kwon2020uncertainty,shridhar2019bcnn}:
{
\begin{equation}
     \underbrace{\underbrace{\frac{1}{M} \sum_{m=1}^M (P_m - \Bar{P})^T  (P_m - \Bar{P})}_{\text{Epistemic unc. ($\Unce$)}} + \underbrace{\frac{1}{M} \sum_{m=1}^M diag(P_m) - P_m^T \cdot P_m}_{\text{Aleatoric unc. ($\Unca$)}}}_{\text{Predictive uncertainty ($\Uncp$)}},
\end{equation}
}
{\noindent} where $\Bar{P} = \frac{1}{M} \sum_{m=1}^M P_m$ and $P_m = P(Y|X=\mathbf{x})$ of the $m^{th}$ Monte Carlo sample with $M$ being the number of Monte Carlo samples.
To obtain group-wise uncertainty estimations, we aggregate the quantified uncertainty values for the samples of that group by averaging.

\subsection{Uncertainty-based Group Fairness Measures}
We now introduce our novel fairness notion based on averaged predictive uncertainty over groups where each group is defined by a set of sensitive attributes. For this, we use the uncertainty types and their quantification as outlined in Section \ref{sect:uncertainty_and_quantification} and extend the definition of fairness in Section \ref{sect:measuring_fairness}.  

\begin{definition}[\textsc{Uncertainty-Fairness Measure}]

A model is fair if its uncertainties are the same across different groups. More formally, extending the definition in Section \ref{sect:measuring_fairness}
\begin{equation}
\textrm{Fair}(f; \Unc, D) \equiv \Unc(D, f, G = 0) = \Unc(D, f, G=1),
\end{equation}
where $\Unc$ is an uncertainty measure, e.g., predictive uncertainty ($\Uncp$), epistemic uncertainty ($\Unce$), or aleatoric uncertainty ($\Unca$) as introduced in Section \ref{sect:uncertainty_and_quantification}.
\end{definition}

\begin{proposition}[\textsc{Independence of Uncertainty Fairness}]
Consider a predictor $f(\cdot; \theta)$ with point-predictions $\{\hat{y}_i\}_i$ (and associated probabilities $\{P(\hat{y}_i | \mathbf{x}_i)\}_i$) and uncertainties $\{\Unc_i\}_i$ (namely, predictive, epistemic and aleatoric). Then, uncertainty fairness $\textrm{Fair}(f; \Unc, D)$ is independent to the conventional point-measure based fairness $\textrm{Fair}(f; \mathcal{M}, D)$. More formally:
\begin{itemize}
    \item $\textrm{Fair}(f; \mathcal{M}, D) \centernot\implies \textrm{Fair}(f; \Unc, D) $.
    \item $\textrm{Fair}(f; \Unc, D) \centernot\implies \textrm{Fair}(f; \mathcal{M}, D)$.
\end{itemize}
$\textrm{Fair}(f; \Unc, D)$ does not imply $\textrm{Fair}(f; \mathcal{M}, D)$ or vice versa.

\end{proposition}

\begin{proof}
We will prove the two non-implications in the proposition using contradictions:

\textbf{Proof of $\textrm{Fair}(f; \mathcal{M}, D) \centernot\implies \textrm{Fair}(f; \Unc, D)$}: we assume that the implication is true, i.e., $\textrm{Fair}(f; \mathcal{M}, D) \implies \textrm{Fair}(f; \Unc, D)$. That means there may not be a predictor $f$ which is $\mathcal{M}$-wise fair but $\Unc$-wise unfair. 
As contradiction, we select as examples the Synthetic Dataset 1 \& 2 in Section \ref{sect:datasets}(A,B) -- see also Fig. \ref{fig:synthetic_1_and_2}. In this contradictory example, we see a predictor (namely, a BNN -- see Section \ref{sect:training_details} for architecture and training details) which is 
$\mathcal{M}$-wise fair but $\Unc$-wise unfair (Table \ref{tab:SD_results}). Therefore, $\textrm{Fair}(f; \mathcal{M}, D) \implies \textrm{Fair}(f; \Unc, D)$ is not necessarily true, and therefore, $\textrm{Fair}(f; \mathcal{M}, D) \centernot\implies \textrm{Fair}(f; \Unc, D)$.

\textbf{Proof of $\textrm{Fair}(f; \Unc, D) \centernot\implies \textrm{Fair}(f; \mathcal{M}, D)$}. We will follow the same reasoning for this non-implication: we assume that the implication is true, i.e., $\textrm{Fair}(f; \mathcal{U}, D) \implies \textrm{Fair}(f; \mathcal{M}, D)$. That means there may not be a predictor $f$ which is $\Unc$-wise fair but $\mathcal{M}$-wise unfair.
As contradiction, we select the example in Sect. \ref{sect:datasets}(C) -- see also Fig. \ref{fig:syn_dataset3}. In this  example, we see a predictor (again, a BNN -- see Sect. \ref{sect:training_details} for architecture and training details) which is 
$\Unc$-wise fair but $\mathcal{M}$-wise unfair (Table \ref{tab:SD_results}). Therefore, $\textrm{Fair}(f; \mathcal{U}, D) \implies \textrm{Fair}(f; \mathcal{M}, D)$ is not necessarily true, and therefore, $\textrm{Fair}(f; \Unc, D) \centernot\implies \textrm{Fair}(f; \mathcal{M}, D)$.
\end{proof}

\subsection{Uncertainty-based Individual Fairness}
\label{sect:indv_fairness}

We extend the definition in Eq. \ref{eqn:consistency} to account for ``similar individuals should have similar prediction \textbf{uncertainties}'':
\begin{equation}\label{eqn:indv_fairness}
   \mathcal{F}^{indv}_\Unc(X=\mathbf{x}_i) = 1-\left| \Unc_i - \frac{1}{k} \sum_{\mathbf{x}_j \in k\textrm{NN}(\textbf{x}_i)} \Unc_j \right|,
\end{equation}
which we aggregate over a group by averaging.

\section{Experiments}

In this section, we introduce the datasets used, the implementation and training details as well as the evaluation measures used within the experiments.

\subsection{Datasets}
\label{sect:datasets}

We introduce three synthetic datasets and utilize three real datasets to evaluate the measures. 
We adopt the approach of \cite{zafar2017fairness} for all synthetic dataset curation.
Each synthetic dataset has 320 samples with 20\% 
reserved for testing. 

\begin{figure*}[hbt!]
  \centerline{
  \subfigure[Synthetic Dataset 1]
  {
    \includegraphics[width=0.3\columnwidth]{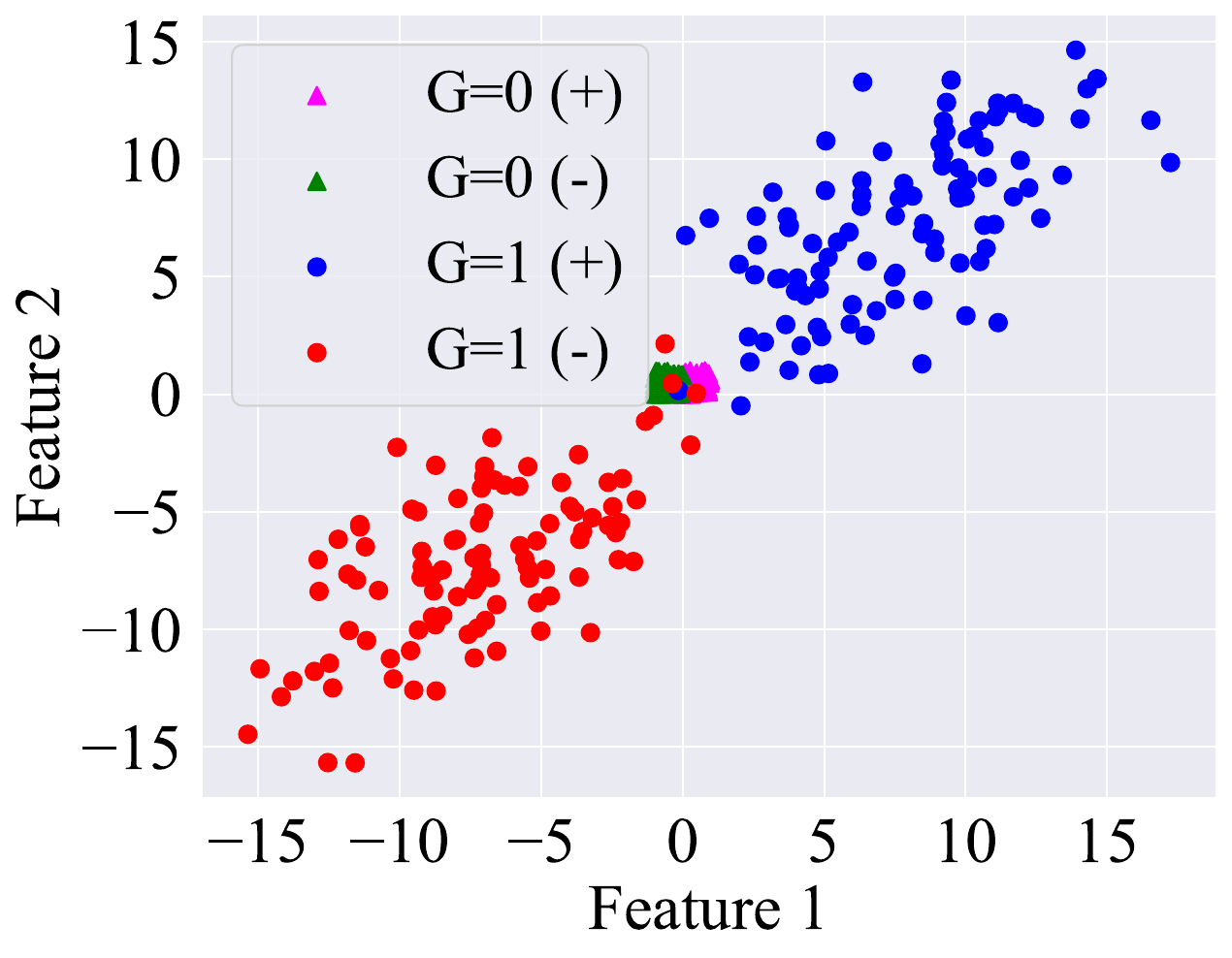}
    \label{fig:aleatoric_toy_dataset}
  }
  \subfigure[Synthetic Dataset 2]
  {
    \includegraphics[width=0.3\columnwidth]{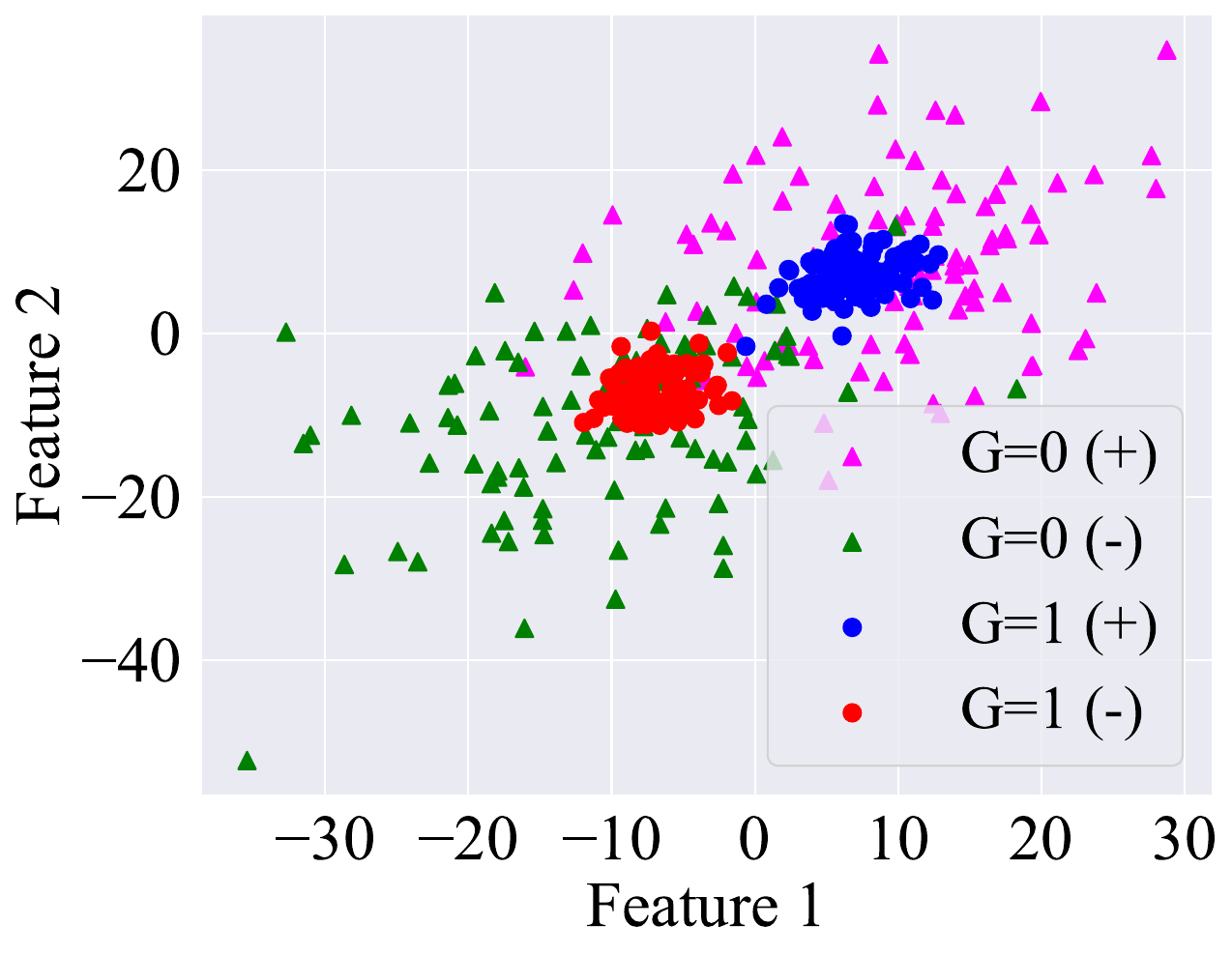}
    \label{fig:epistemic_toy_dataset}
  }
  \subfigure[Synthetic Dataset 3]
  {
    \includegraphics[width=0.3\columnwidth]{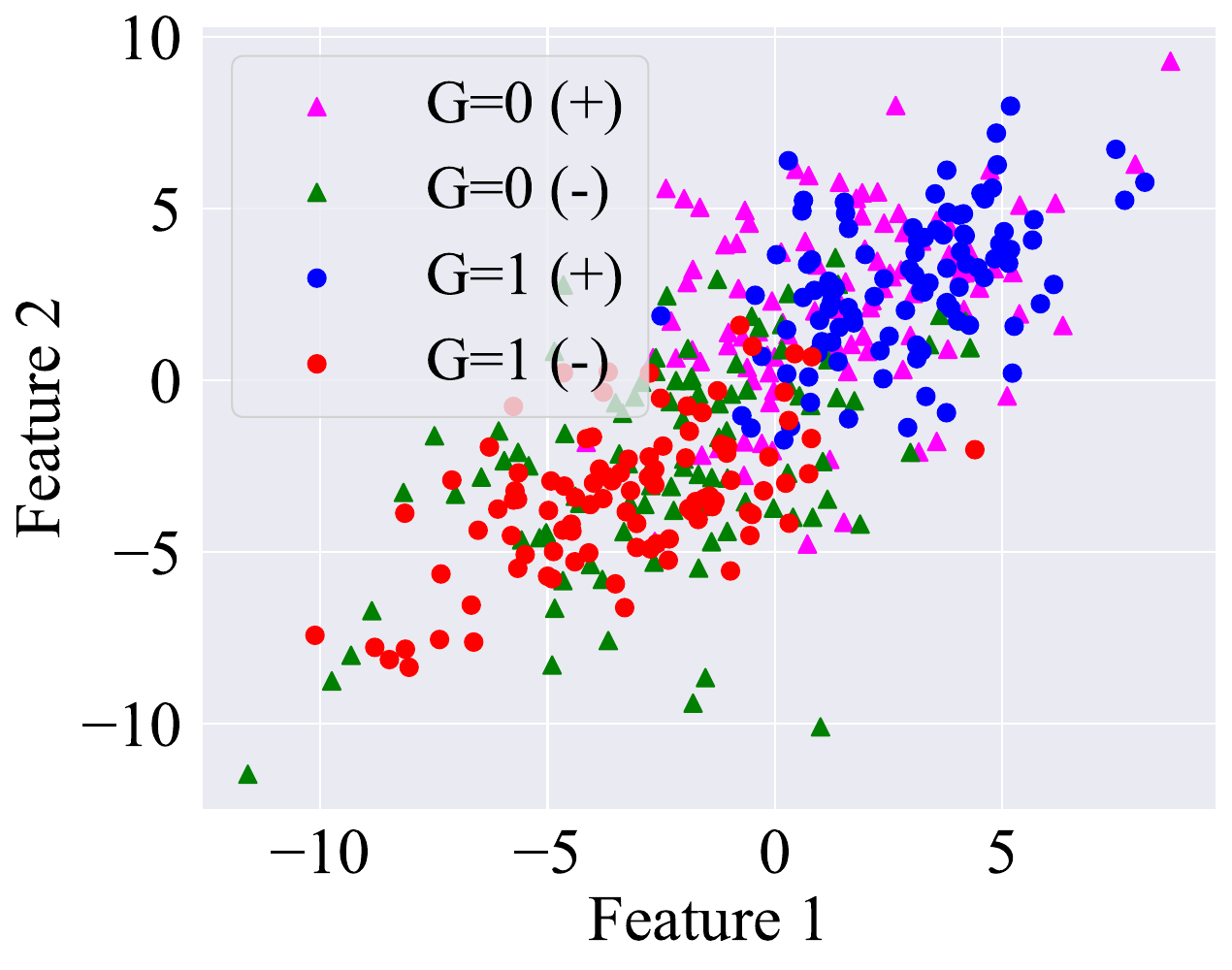}
    \label{fig:syn_dataset3}
 }
 }
\caption{Two datasets that appear to be fair with point-based measures but unfair in terms of \textbf{(a)} aleatoric uncertainty and \textbf{(b)} epistemic uncertainty. In \textbf{(c)}, we see a set where the classifier is fair in terms of uncertainties (both epistemic and aleatoric) but unfair in terms of point-based measures.
}
\label{fig:synthetic_1_and_2}
\vspace{-0.5cm}
\end{figure*}

\subsubsection{(A) Synthetic Dataset 1 (SD1) - Case of Aleatoric Uncertainty}
This dataset highlights how a classifier may appear to be fair in terms of point-based performance metrics yet unfair \textit{in terms of aleatoric uncertainties}. 
We obtain 100 samples from each of the following four multivariate distributions, one for each of the respective attribute-label pair that we consider:
{
\begin{eqnarray}
    P(X | G = 0, Y = 0) & = & \mathcal{B}eta (\alpha=[0.5, 0.5], \beta=[0.5, 0.5]), \\ 
    P(X | G = 0, Y = 1) & = & -\mathcal{B}eta (\alpha=[0.5, 0.5], \beta=[0.5, 0.5]), \\  
    P(X | G = 1, Y = 0) & = & \mathcal{N} ([-7, -7], [15,10 ; 10, 15]), \\ 
    P(X | G = 1, Y = 1) & = & \mathcal{N} ([7, 7], [15,10 ; 10, 15]).
\end{eqnarray}
}

\subsubsection{(B) Synthetic Dataset 2 (SD2) - Case of Epistemic Uncertainty}
This dataset highlights how a classifier may be fair with point-based fairness measures but \textit{unfair in terms of epistemic uncertainties}. 
We obtain 100 samples from each of the following four multivariate distributions, one for each of the respective attribute-label pair that we consider:
{
\begin{eqnarray}
    P(X | G = 0, Y = 0) & = & \mathcal{N} ([-10, -10], [100,30 ; 30, 100]), \\ 
    P(X | G = 0, Y = 1) & = & \mathcal{N} ([10, 10], [100, 30 ; 30, 100]), \\  
    P(X | G = 1, Y = 0) & = & \mathcal{N} ([-7, -7], [5, 1 ; 5, 1]), \\ 
    P(X | G = 1, Y = 1) & = & \mathcal{N} ([7, 7], [5, 1 ; 1, 5]).
\end{eqnarray}
}

\subsubsection{(C) Synthetic Dataset 3 (SD3) - Fair in Uncertainty, Unfair in Predictions}
This dataset aims to highlight how a classifier may be fair according to the proposed uncertainty-based fairness measures but \textit{unfair in terms of point-based fairness measures}. 
We obtain 100 samples from each of the following four multivariate distributions, one for each of the respective attribute-label pair that we consider:
{
\begin{eqnarray}
    P(X | G = 0, Y = 0) & = & \mathcal{N} ([-2, -2], [7, 3 ; 3, 7]), \\ 
    P(X | G = 0, Y = 1) & = & \mathcal{N} ([2, 2], [7, 3 ; 3, 7]), \\  
    P(X | G = 1, Y = 0) & = & \mathcal{N} ([-3, -3], [5, 3 ; 5, 3]), \\ 
    P(X | G = 1, Y = 1) & = & \mathcal{N} ([3, 3], [5, 3 ; 3, 5]).
\end{eqnarray}
}

\subsubsection{(D) COMPAS Recidivism Dataset } The COMPAS Recidivism Dataset is a dataset with criminal offenders' records generally used to predict recidivism (binary classification) \cite{angwin2016machine}. The data contains 6172 samples with 14 features. 
We follow \cite{zafar2017fairness} in terms of the considered attributes and dataset splits.
We assume that the positive label ($Y=1$) stands for the cases where the subject has recidivated and vice versa.

\subsubsection{(E) Adult Income Dataset } The Adult Income Dataset contains a $48K+$ samples with 14 features \cite{becker1996adult}.
The task is to predict if a person's annual income is greater than $\$50K$ ($Y=1$) or not.
We do not consider the samples with missing entries, resulting in a total of $45K+$ samples.
We adhere to the training-testing split provided by the authors. 

\subsubsection{(F) D-Vlog Depression Detection Dataset } 
The D-Vlog Depression Detection Dataset contains visual and acoustic features from Youtube videos of 555 depressed and 406 non-depressed samples belonging to 639 females and 322 males \cite{yoon2022dvlog}. The authors truncated the videos with longer than $t= 596s$ and zero-pad shorter ones. D-Vlog only provides the gender attribute for its samples. We follow the training and testing splits as provided by the authors. We assume that the positive label ($Y=1$) stands for the depressed class and vice versa.

\subsection{Implementation and Training Details}
\label{sect:training_details}


For all of our experiments, except for D-Vlog,
we utilize Bayesian Neural Networks (BNNs) for both classification and uncertainty estimation. 
To address the intractability of $P(Y | X)$, we utilize the well-known \textit{Bayes by Backprop} method \cite{blundell2015weight}
which minimizes the following objective consisting of a KL divergence \cite{kullback1951information} term and a numerically-stable negative log-likelihood term as proposed in \cite{kendall2017uncertainties}:
{\small
\begin{equation}
    \mathcal{L}(\theta) = \sum_{m=1}^M \underbrace{ [\log q_{\theta} (\omega_m) - \log P (\omega_m)}_{\text{KL divergence}} + \lambda \underbrace{\mathcal{L}_{NLL}(\hat{Y}, Y)}_{\text{classification loss}}],
\end{equation}
}
with $\lambda$ being a constant.
%
%
%

For all experiments, we use the Adam optimizer \cite{kingma2017adam}. Following \cite{kwon2020uncertainty}, we set $T=10$ (the number of Monte Carlo samples for uncertainty quantification as defined in Section 
4.1).
Furthermore, following one of the settings provided in \cite{blundell2015weight}, we use $10$ Monte Carlo samples to approximate the variational posterior, $q_{\theta}(\omega)$, and sample the initial mean of the posterior from a Gaussian with $\mu = 0$ and $\sigma = 1$.
The $\pi$ value, weighting factor for the prior, is set to $0.5$ and the two $\sigma_1$ and $\sigma_2$ values for the scaled mixture of Gaussians is set to $0$ and $6$ respectively.
We consider $\lambda$ from the BNN training objective to be $2000$.
We utilize early stopping to determine the number of training iterations for all experiments.
In the following, we describe dataset specific details. In all cases, the hyper-parameters are tuned to avoid over-fitting:

\paragraph{Synthetic Datasets:} As the datasets are relatively simple, we observe that BNNs with no hidden layers suffice for all three synthetic datasets.
We train all of them for $5$ epochs with a batch size of $8$.

\paragraph{COMPAS Recidivism Dataset:} We employ a BNN with a single hidden layer of size $100$.
We train the model for $10$ epochs with a batch size of $256$.
Similar to \cite{chouldechova2017fair}, 
we consider fairness with respect to race, gender and age. 
For the race attribute, we follow \cite{zafar2017fairness} and focus on the fairness gap between black and white subgroups, considering African-Americans as the minority group, $G_0$.
For the gender attribute, we designate females as the minority group ($G_0$) due to the class imbalance in favor of the male group. For the age attribute, we consider individuals younger than $25$ to be the minority group and individuals older than $45$ as the majority group since our classifier provided the worst result for those younger than 25 and overall best results for those older than $45$. To keep the coverage to a binary setting, we do not consider individuals aged between $25$ and $45$. Extending the measures to such a multi-valued setting is straightforward \cite{xu2020investigating} and left as future work. 
%
%

\paragraph{Adult Income Dataset:} We employ a BNN with no hidden layers where the intermediate size is $25$.
We train the model for $5$ epochs with a batch size of $256$.
Although a deeper analysis could be conducted through considering other variables such as marriage status, highest education level, occupation and nationality, for the sake of consistency with the analysis of the other datasets, we limit the experiments to race, gender and age.

\paragraph{D-Vlog Depression Detection Dataset:} D-Vlog samples have significantly larger dimensionality (596s of 136-dim visual and 25-dim acoustic features) compared to COMPAS and Adult, which turned out to be challenging for BNNs. Therefore, we utilize the transformer-based \textit{Depression Detector} architecture proposed by \cite{yoon2022dvlog}. For uncertainty estimation, we follow \cite{lakshminarayanan2017simple} to obtain $T$ predictions with an ensemble of $T$ models and use the same method (Eq. 
4.1) as with uncertainty estimation with BNNS.
Specifically, instead of performing MC forward passes, we train $T$ different models on the same training set and consider their predictions in the same testing set during the uncertainty quantification process. We choose $T=5$ 
as existing work indicates that performance tends to peak at that number \cite{havasi2020training}.

For all of training configurations, we directly use the  setting of \cite{yoon2022dvlog} with a learning rate of $0.0002$ and a batch size of $32$, optimized for $50$ epochs through the Adam optimizer \cite{kingma2017adam}.
For the dropout rate, we empirically choose $0.1$ though it was not explicitly provided by the authors in the original work. For more details on the architecture and the relevant training details, we refer the reader to \cite{yoon2022dvlog}.


\subsection{Evaluation Measures}
\label{sect:performance_evaluation}

We evaluate classification performance in terms of accuracy ($\PMeasure_{Acc}$), Positive Predictive Value ($\PMeasure_{PPV}=TP/(TP+FN)$), Negative Predictive Value ($\PMeasure_{NPV}=TN/(TN+FN)$), False Positive Rate ($\PMeasure_{FPR}$) and False Negative Rate ($\PMeasure_{NPR}$). Fairness measures ($\FMeasure$) are defined as follows (similar to e.g. \cite{feldman2015certifying,xu2020investigating,cheong2024fairrefuse}): 
{ 
\begin{eqnarray}
    \textrm{Statistical Parity: } & \FMeasure_{SP} = \frac{P(\hat{Y}=1|G=0)}{P(\hat{Y}=1|G=1)}, \\
    \textrm{Equal Opportunity: } & \FMeasure_{EOpp} = \frac{P(\hat{Y}=0|Y=1, G=0)}{P(\hat{Y}=0|Y=1, G=1)}, \\
    \textrm{Equalized Odds: } & \FMeasure_{EOdd} = \frac{P(\hat{Y}=1|Y=y, G=0)}{P(\hat{Y}=1|Y=y, G=1)}, \\
    \textrm{Equal Accuracy: } & \FMeasure_{EAcc} = \frac{\PMeasure_{Acc}(D, f, G=0)}{\PMeasure_{Acc}(D, f, G=1)}, \\
    \textrm{Uncertainty Fairness: } & \FMeasure_{u} = \frac{\Unc_{u}(D, f, G=0)}{\Unc_{u}(D, f, G=1)},
\end{eqnarray}}
where $u$ can be $Alea$ (Aleatoric), $Epis$ (Epistemic) or $Pred$ (Predictive).

\section{Results}
In this section, we discuss the results obtained across both the synthetic and real-world datasets.

\subsection{Experiment 1: Synthetic Datasets}
Here, we analyze the point-based and uncertainty-based fairness measures with SD1, SD2 and SD3.

\begin{table}[htb!]
\caption{Experiment 1: The analysis with SD1, SD2 and SD3 datasets. We see that, for both SD1 and SD2, the classifier is fair in terms of point-based measures ($|\FMeasure - 1| \le 0.2$ -- following \cite{feldman2015certifying}) whereas it is unfair in terms of aleatoric uncertainty for SD1 and epistemic uncertainty unfair for SD2. We see the inverse for the SD3 dataset. Unfair values are \highlight{highlighted}.
    \label{tab:SD_results}}
  \centering\small
  \begin{tabular}{ccccccc}
    \toprule
    & \multicolumn{2}{c}{SD1} & \multicolumn{2}{c}{SD2} & \multicolumn{2}{c}{SD3}\\
    \midrule
    \textbf{Measure} & \textbf{$G_0$} & \textbf{$G_1$} & \textbf{$G_0$} & \textbf{$G_1$} & \textbf{$G_0$} & \textbf{$G_1$} \\
    \midrule
    \multicolumn{5}{l}{\it\scriptsize Performance Measures}\\
    $\uparrow$ $\PMeasure_{Acc}$          & 0.95   & 0.95  & 0.95   & 0.95   & 0.74 & {0.93}\\
    $\uparrow$ $\PMeasure_{PPV}$          & {0.95}   & 0.90  & 0.95   & 0.95 & 0.62      & {0.96}\\
    $\uparrow$ $\PMeasure_{NPV}$          & 0.94   & {0.95}  & 0.94   & 0.94 & {0.93}      & 0.91\\
    $\downarrow$ $\PMeasure_{FPR}$        & 0.06   & {0.05} & {0.05}   & 0.06 & 0.38      & {0.04}\\
    $\downarrow$ $\PMeasure_{FNR}$        & 0.05   & 0.05 & 0.05   & 0.05 & {0.07}      & 0.08\\
    $\downarrow$ $\mathcal{\Unc}_e$  & 0.0001   & 0.0001 & 0.0011   & {0.0004}  & 0.0002      & 0.0002\\
    $\downarrow$ $\mathcal{\Unc}_a$  & 0.4926   & {0.1053}  & {0.1915}   & 0.2193 & 0.3349   & {0.3229}\\
    $\downarrow$ $\mathcal{\Unc}_p$ & 0.4927   & {0.1054} & {0.1926}   & 0.2197 &  0.3351    & {0.3231}\\
    \hdashline[1pt/1pt]
    \multicolumn{5}{l}{\it\scriptsize Point-based Fairness Measures}\\
    $\FMeasure_{SP}$ & \multicolumn{2}{c}{1.07} & \multicolumn{2}{c}{1.00} & \multicolumn{2}{c}{1.17}\\
    $\FMeasure_{Opp}$ & \multicolumn{2}{c}{1.00} & \multicolumn{2}{c}{1.00} & \multicolumn{2}{c}{1.01}\\
    $\FMeasure_{Odd}$ & \multicolumn{2}{c}{1.05} & \multicolumn{2}{c}{0.95} & \multicolumn{2}{c}{\highlight{7.90}}\\
    $\FMeasure_{EAcc}$ & \multicolumn{2}{c}{1.00} & \multicolumn{2}{c}{1.00} & \multicolumn{2}{c}{\highlight{0.79}}\\
    \hdashline[1pt/1pt]
    \multicolumn{5}{l}{\it\scriptsize Uncertainty-based Fairness Measures (Ours)}\\
    {$\FMeasure_{Epis}$} & \multicolumn{2}{c}{{1.01}} & \multicolumn{2}{c}{{\highlight{2.75}}} & \multicolumn{2}{c}{{1.05}}\\
    {$\FMeasure_{Alea}$} & \multicolumn{2}{c}{{\highlight{4.68}}} & \multicolumn{2}{c}{{0.87}} & \multicolumn{2}{c}{{1.04}}\\
    {$\FMeasure_{Pred}$} & \multicolumn{2}{c}{{\highlight{4.67}}} & \multicolumn{2}{c}{{0.88}} & \multicolumn{2}{c}{{1.04}}\\
    \bottomrule
  \end{tabular}
\end{table}

\paragraph{Analysing $\textrm{Fair}(f; \Unc, D) \centernot\implies \textrm{Fair}(f; \mathcal{M}, D)$.}
With reference to SD 1 \& 2 as introduced in Section \ref{sect:datasets} and 
in Fig. \ref{fig:aleatoric_toy_dataset} and \ref{fig:epistemic_toy_dataset}, we select the group with higher uncertainty estimations as the minority group, i.e., $G_0$.
From Table \ref{tab:SD_results}, we see that the classifier (BNN) can solve the classification task with a good level of performance (with high accuracy and low mis-classification). Moreover, the classifier is fair in terms of the widely-used point-based measures ($\FMeasure_{SP}, \FMeasure_{Opp}, \FMeasure_{Odd}$ and $\FMeasure_{EAcc}$)
with $|\FMeasure - 1| \le 0.2$ -- following \cite{feldman2015certifying}.
However, our uncertainty-based measures suggest that the classifier is significantly unfair in terms of aleatoric uncertainty (for SD1 with $\FMeasure_{Alea} = 4.68$) and epistemic uncertainty (for SD2 with $\FMeasure_{Epis} = 275$).

\paragraph{Analysing $\textrm{Fair}(f; \mathcal{M}, D) \centernot\implies \textrm{Fair}(f; \Unc, D)$}
With reference to SD 3 as introduced in Section \ref{sect:datasets} and 
and Fig. \ref{fig:syn_dataset3}, we select the group with the lower classification performance as the minority group, $G_0$.
Results in Table \ref{tab:SD_results} suggest that the classifier provides a good level of performance for the majority group ($G_1$) and that the classifier is unfair in terms of some of the point-based fairness measures (namely, $\FMeasure_{Odd}$ = $7.9$ and $\FMeasure_{EAcc}$ = $0.79$). However, the classifier appears to be fair across the uncertainty-based fairness measures.

\begin{table}[hbt!]
  \caption{Experiment 2: The analysis with COMPAS. 
    Unfair values ($|\FMeasure-1|>0.2$, following \cite{feldman2015certifying}) are \highlight{highlighted}.  B/W: Black/White. F/M: Female/Male.
    \label{tab:compas_fairness_results}}
  \centering\small
  \begin{tabular}{ccccc}
    \toprule
    \textbf{Measure} & \textbf{B ($G_0$)} & \textbf{W ($G_1$)} & \textbf{F ($G_0$)} & \textbf{M ($G_1$)}\\
    \midrule
    $\uparrow$ Sample Size   & {3175}   & 2103    & 1175 & {4997}\\
    \hdashline[1pt/1pt]
    \multicolumn{5}{l}{\it\scriptsize Performance Measures}\\
    $\uparrow$ $\PMeasure_{Acc}$   & {0.70}   & 0.68    & {0.77} & 0.68\\
    $\uparrow$ $\PMeasure_{PPV}$        & {0.69}   & 0.68    & {0.70} & 0.68 \\
    $\uparrow$ $\PMeasure_{NPV}$        & {0.72}   & 0.68     & {0.78} & 0.68\\
    $\downarrow$ $\PMeasure_{FPR}$      & 0.34   & {0.10}     & {0.05} & 0.25\\
    $\downarrow$ $\PMeasure_{FNR}$      & {0.25}   & 0.66     & 0.67 & {0.39}\\
    $\downarrow$ $\mathcal{\Unc}_e$  & 0.0006 & {0.0004}  & {0.0003} & 0.0006\\
    $\downarrow$ $\mathcal{\Unc}_a$  & 0.2299 & {0.1578}  & {0.1599} & 0.2053\\
    $\downarrow$ $\mathcal{\Unc}_p$ & 0.2305 & {0.1583}   & {0.1602} & 0.2059\\
    \hdashline[1pt/1pt]
    \multicolumn{5}{l}{\it\scriptsize Point-based Fairness Measures}\\
    $\FMeasure_{SP}$ & \multicolumn{2}{c}{\highlight{2.84}}    & \multicolumn{2}{c}{\highlight{0.31}}\\
    $\FMeasure_{EOpp}$ & \multicolumn{2}{c}{\highlight{2.19}}  & \multicolumn{2}{c}{\highlight{0.54}} \\
    $\FMeasure_{EOdd}$ & \multicolumn{2}{c}{\highlight{1.57}}  & \multicolumn{2}{c}{\highlight{0.40}} \\
    $\FMeasure_{EAcc}$ & \multicolumn{2}{c}{1.03} & \multicolumn{2}{c}{1.13}  \\
    \hdashline[1pt/1pt]
    \multicolumn{5}{l}{\it\scriptsize Uncertainty-based Fairness Measures (Ours)}\\
    {$\FMeasure_{Epis}$} & \multicolumn{2}{c}{\highlight{1.55}} & \multicolumn{2}{c}{\highlight{0.50}}\\
    $\FMeasure_{Alea}$ & \multicolumn{2}{c}{\highlight{1.46}}  & \multicolumn{2}{c}{\highlight{0.78}}\\
    $\FMeasure_{Pred}$ & \multicolumn{2}{c}{\highlight{1.46}}  & \multicolumn{2}{c}{\highlight{0.78}}\\
    \bottomrule
  \end{tabular}
\end{table}

\subsection{Experiment 2: Real-world Datasets}
\label{sect:experiments_real_world}

In this section, we analyze the fairness measures across the real-world datasets, i.e., COMPAS, Adult and D-Vlog.
%

\subsubsection{The COMPAS Dataset.}

With reference to Table  \ref{tab:compas_fairness_results}, across \textbf{race},
results suggest that there is strong bias against African-Americans in terms of recidivism even though there are more samples for African-Americans: The classifier has a clear tendency to suggest a black person to recidivate ($\PMeasure_{FPR}$ = $0.34$ African-Americans vs. $0.10$) and vice versa for Whites ($\PMeasure_{FNR}$ = $0.66$ for Whites vs. $0.25$). The point-based fairness measures (except for $\PMeasure_{Eacc}$) capture this bias strongly, so do the uncertainty-based measures. Despite having more samples, African-Americans have higher $\Unc_e$ and $\Unc_a$, leading to significant unfairness in terms of uncertainty ($\FMeasure_{Epis}$ and $\FMeasure_{Alea}$).

Across \textbf{gender},
females have significantly better prediction performance compared to Males, with the exception of $\PMeasure_{FNR}$ ($M_{FNR}$ = $0.67$ for Females vs. $0.39$ for Males). This suggests that the classifier is biased to predict $Y=0$ (``no recidivism") for Females. Both point-based and uncertainty-based fairness measures capture this bias against Males.
Across epistemic uncertainty, we 
hypothesize that the classifier is less certain for Males. However, fairness gaps in terms of aleatoric uncertainty ($0.78$) and predictive uncertainty ($0.78$) are close to the acceptable fairness boundary ($0.8$), suggesting that the main issue across gender may be the sample imbalance problem across groups (see also Table \ref{tab:real_dataset_distributions}).

\begin{table*}[hbt!]
\caption{Experiment 2: The analysis with the COMPAS Recidivism Dataset for race (Black vs. White), age (younger than 25 vs. older than 45) and gender (male vs. female) attributes. Severe values of fairness values ($|\FMeasure-1|>0.2$, following \cite{zanna2022bias}) are \highlight{highlighted}. \label{suptab:compas_fairness_results}
    }
  \centering\small
  \begin{tabular}{cccccccc}
    \toprule
                    & \multicolumn{2}{c}{\textbf{Race}} & \multicolumn{3}{c}{\textbf{Age}} & \multicolumn{2}{c}{\textbf{Gender}}\\
    \cdashline{2-8}
    \textbf{Measure} & \textbf{Black} & \textbf{White} & \textbf{$<$25 } & \textbf{25-45} & \textbf{$>$45} & \textbf{Female} & \textbf{Male}\\
        & ($G_0$) & ($G_1$) & ($G_0$) & ($G_1$) & ($G_0$) & ($G_1$)\\
    \midrule
    $\uparrow$ Sample Size   & {3175}   & 2103   & 1347  & {3532}   & 1293   & 1175 & {4997}\\
    \hdashline[1pt/1pt]
    \multicolumn{7}{l}{\it\scriptsize Performance Measures}\\
    $\uparrow$ $\PMeasure_{Acc}$   & {0.70}   & 0.68   & 0.64  & 0.71   & {0.75}   & {0.77} & 0.68\\
    $\uparrow$ $\PMeasure_{PPV}$        & {0.69}   & 0.68   & 0.64  & {0.72}   & 0.64   & {0.70} & 0.68 \\
    $\uparrow$ $\PMeasure_{NPV}$        & {0.72}   & 0.68   & 0.63  & 0.70   & {0.78}   & {0.78} & 0.68\\
    $\downarrow$ $\PMeasure_{FPR}$      & 0.34   & {0.10}   & 0.41  & 0.18   & {0.12}   & {0.05} & 0.25\\
    $\downarrow$ $\PMeasure_{FNR}$      & {0.25}   & 0.66   & {0.32}  & 0.44   & 0.54    & 0.67 & {0.39}\\
    $\downarrow$ $\mathcal{\Unc}_e$  & 0.0006 & {0.0004} & 0.0010 & 0.0005 & {0.0002}  & {0.0003} & 0.0006\\
    $\downarrow$ $\mathcal{\Unc}_a$  & 0.2299 & {0.1578} & 0.3459 & 0.1712 & {0.1027}  & {0.1599} & 0.2053\\
    $\downarrow$ $\mathcal{\Unc}_p$ & 0.2305 & {0.1583} & 0.3469 & 0.1717 & {0.1029}  & {0.1602} & 0.2059\\
    \hdashline[1pt/1pt]
    \multicolumn{7}{l}{\it\scriptsize Point-based Fairness Measures}\\
    $\FMeasure_{SP}$ & \multicolumn{2}{c}{\highlight{2.84}} & \multicolumn{3}{c}{\highlight{2.44}} & \multicolumn{2}{c}{\highlight{0.31}}\\
    $\FMeasure_{EOpp}$ & \multicolumn{2}{c}{\highlight{2.19}}  & \multicolumn{3}{c}{\highlight{1.48}} & \multicolumn{2}{c}{\highlight{0.54}} \\
    $\FMeasure_{EOdd}$ & \multicolumn{2}{c}{\highlight{1.57}}  & \multicolumn{3}{c}{\highlight{2.37}} & \multicolumn{2}{c}{\highlight{0.40}} \\
    $\FMeasure_{EAcc}$ & \multicolumn{2}{c}{1.03} & \multicolumn{3}{c}{0.85} & \multicolumn{2}{c}{1.13}  \\
    \hdashline[1pt/1pt]
    \multicolumn{7}{l}{\it\scriptsize Uncertainty-based Fairness Measures (Ours)}\\
    {$\FMeasure_{Epis}$} & \multicolumn{2}{c}{\highlight{1.55}}  & \multicolumn{3}{c}{\highlight{4.35}} & \multicolumn{2}{c}{\highlight{0.50}}\\
    $\FMeasure_{Alea}$ & \multicolumn{2}{c}{\highlight{1.46}} & \multicolumn{3}{c}{\highlight{3.36}}  & \multicolumn{2}{c}{\highlight{0.78}}\\
    $\FMeasure_{Pred}$ & \multicolumn{2}{c}{\highlight{1.46}}  & \multicolumn{3}{c}{\highlight{3.37}} & \multicolumn{2}{c}{\highlight{0.78}}\\
    \bottomrule
  \end{tabular}
    
\end{table*}

In Table \ref{suptab:compas_fairness_results}, we see the complete table for COMPAS, including the age attribute.
Across age, we observe that almost all of the performance metrics are better for those age greater than $45$ compared to the others, with the exception of $\PMeasure_{PPV}$ and $\PMeasure_{FNR}$.
For $\PMeasure_{PPV}$, samples with ages between $25-45$ is the best with $\PMeasure_{PPV} = 0.72$ and for $\PMeasure_{FNR}$, samples with ages less than $25$ is the best with  $\PMeasure_{FNR} = 0.32$.
As explained within Section 
\ref{sect:experiments_real_world}, we compute both the point-based and proposed uncertainty-based measures by considering the subgroup age greater than $45$ as the majority subgroup ($G_1$) and the subgroup age less than $25$ as the minority subgroup ($G_0$).

Across the point-based fairness measures, we observe that all but one of the measures point to \textit{unfair} predictions, with $\FMeasure_{SP} = 2.44$, $\FMeasure_{Opp} = 1.48$ and $\FMeasure_{Odd} = 2.37$.
Similar to the race attribute, $\FMeasure_{EAcc} = 0.85$ claims \textit{fair} predictions, with all of the other point-based measures directly implying that the model in question is inclined to \textit{unfairly predict} the subgroup age less than $25$ as $y = 1$, i.e recidivating an offense.

Furthermore, the proposed uncertainty-based fairness measures also show similar results with $\FMeasure_{Epis} = 4.35$, $\FMeasure_{Alea} = 3.36$ and $\FMeasure_{Pred} = 3.37$.
A similar conclusion with the race attribute could be arrived here with the $\FMeasure_{Epis} = 4.35$, i.e the \textit{lack of data} according to the model behavior is higher for the subgroup age less than $25$ compared to the the subgroup age greater than $45$ even though the dataset actually contains less samples for the latter.
Even though we also observe a similar pattern with the race attribute with $\FMeasure_{Alea} = 3.36$ and $\FMeasure_{Pred} = 3.37$, the fairness gap in this case is significantly higher.
These measures also show that the classification hardness, the noise faced by the model according to its own behavior, is drastically higher for the subgroup of individuals with an age of less than $25$.

\paragraph{Social Impact:}
From the results above, we see how existing point-based measures merely highlight the prediction bias present. 
Our proposed uncertainty-based measures go a step beyond by serving as a fairness evaluation tool which points towards the \textit{potential source of bias}: The persistent social inequity across race \cite{ding2021retiring}, and towards a potential solution: balancing samples across gender.
In addition, data collected from a real-world setting is bound to be implicated or corrupted by group-dependent labelling or annotation noise \cite{wang2021fair}.
For instance, it has been demonstrated that labels for criminal activity generated via crowdsourcing are systematically biased against certain subgroups \cite{dressel2018accuracy}.
This label class and subgroup dependent heterogeneous systematic bias cannot be quantified by point-based fairness measures.
%
However, we hypothesize that this bias can be captured by our measure which illustrates how the model produced higher $\FMeasure_{Epis}$, 
$\FMeasure_{Alea}$ and $\FMeasure_{Pred}$ for African-Americans and Males despite having more samples for both demographic groups within the training set.
Hence, our measure is a useful diagnostic tool in a real-world setting when clean and accurate labels are not readily available.
%
Future experiments may focus on verifying how unbiased labels may impact the point-based and uncertainty-based fairness measures.

\begin{table} [hbt!]
\caption{Label and sensitive attribute distributions of  COMPAS and Adult. B/W: Black/White. F/M: Female/Male. 
    \label{tab:real_dataset_distributions}}
  \centering\small
  \begin{tabular}{c|cc|cc}
    \toprule
     & \multicolumn{2}{c|}{\textbf{COMPAS}} & \multicolumn{2}{c}{\textbf{Adult}} \\
    \midrule
    \textbf{Group} & $Y=0$ & $Y=1$ & $Y=0$ & $Y=1$ \\
    \midrule
    B   & 1514 (48\%)  & 1661 (52\%)   & 2451 (87\%) & 366 (13\%) \\
    W   & 1281 (61\%)  & 822 (39\%)    & 19094 (74\%) & 6839 (26\%) \\
    \midrule
    F  & 762 (65\%)  & 413 (35\%)     & 8670 (89\%) & 1112 (11\%) \\
    M    & 2601 (52\%)  & 2396 (48\%)   & 13984 (69\%) & 6396 (31\%) \\
   \bottomrule
  \end{tabular}
\end{table}

\begin{table} [htbp]
\caption{Experiment 2: The analysis with Adult. Unfair values ($|\FMeasure-1|>0.2$, following \cite{zanna2022bias}) are \highlight{highlighted}. B/W: Black/White. F/M: Female/Male. \label{tab:adult_results}}
  \centering\small
  \begin{tabular}{ccccc}
    \toprule
    \textbf{Measure} & \textbf{B} ($G_0$) & \textbf{W} ($G_1$) & \textbf{F} ($G_0$) & \textbf{M} ($G_1$) \\
    \midrule
    $\uparrow$ Sample Size          & 2,817  & {25,933}    & 9,872 & {20,380} \\
    \hdashline[1pt/1pt]
    \multicolumn{5}{l}{\it\scriptsize Performance Measures}\\
    $\uparrow$ $\PMeasure_{Acc}$    & {0.86}   & 0.77    & {0.87} & 0.73\\
    $\uparrow$ $\PMeasure_{PPV}$    & 0.40   & {0.60}    & 0.39 & {0.64}\\
    $\uparrow$ $\PMeasure_{NPV}$        & {0.91}   & 0.79    & {0.91} & 0.75\\
    $\downarrow$ $\PMeasure_{FPR}$      & 0.07   & 0.07    & {0.06} & 0.08\\
    $\downarrow$ $\PMeasure_{FNR}$      & {0.67}   & 0.69    & {0.68} & 0.69\\
    $\downarrow$ $\Unce$  & 0.0001 & {6e-7*}  & 0.0001 & {6e-8*}\\
    $\downarrow$ $\Unca$  & 0.01 & {0.01}  & 0.01 & {0.01}\\
    $\downarrow$ $\Unc_p$ & 0.0007 & {0.0004}  & 0.0008 & {0.0003}\\
    \hdashline[1pt/1pt]
    \multicolumn{5}{l}{\it\scriptsize Point-based Fairness Measures}\\
    $\FMeasure_{SP}$ & \multicolumn{2}{c}{{\highlight{0.75}}} & \multicolumn{2}{c}{\highlight{0.62}} \\
    $\FMeasure_{Opp}$ & \multicolumn{2}{c}{1.08} & \multicolumn{2}{c}{1.04} \\
    $\FMeasure_{Odd}$ & \multicolumn{2}{c}{0.87} & \multicolumn{2}{c}{\highlight{0.79}} \\
    $\FMeasure_{EAcc}$ & \multicolumn{2}{c}{1.12} & \multicolumn{2}{c}{1.18} \\
    \hdashline[1pt/1pt]
    \multicolumn{5}{l}{\it\scriptsize Uncertainty-based Fairness Measures (Ours)}\\
    {$\FMeasure_{Epis}$} & \multicolumn{2}{c}{\highlight{151*}} & \multicolumn{2}{c}{\highlight{521*}} \\
    {$\FMeasure_{Alea}$} & \multicolumn{2}{c}{$1.00$} & \multicolumn{2}{c}{$1.00$} \\
    {$\FMeasure_{Pred}$} & \multicolumn{2}{c}{\highlight{1.49}} & \multicolumn{2}{c}{\highlight{2.65}} \\
    \bottomrule
  \end{tabular}
\end{table}


\subsubsection{The Adult Dataset}

As with COMPAS, Black and Female are the minority groups ($G_0$) across race and gender respectively. 
As listed in Table \ref{tab:real_dataset_distributions}, Adult has severe imbalance across labels and groups. 
Across \textbf{race}, Table 
\ref{tab:adult_results} 
shows we only have $2817$ samples for African-Americans vs. $25933$ for Whites. 
However, according to the performance and point-based fairness measures, the fairness gap between the African-Americans and Whites is lower compared to that in COMPAS. 
The uncertainty-based fairness measures provide some interesting insights. Particularly, we observe a surprisingly large fairness gap in terms of epistemic uncertainty, $\FMeasure_{Epis}$ = $151$. This is not surprising since Whites have $10\times$ more samples, yielding very small $\mathcal{\Unc}_e$ value.
Aleatoric uncertainty $\mathcal{\Unc}_a$ values for both groups are very small (compared to all other datasets), which suggest that the classifier has more certainty with respect to data noise, yielding $\FMeasure_{Alea}\sim 1.00$.

The fairness gap also seems lower across \textbf{gender} in Adult compared to COMPAS.
There is also class imbalance across gender, with only $9872$ samples for females vs. $20380$ for males.
We observe conflicting outcomes across the point-based fairness measures: $\FMeasure_{Opp}$ = $1.04$ and $\FMeasure_{EAcc}$ = $1.18$ point to \textit{fair} classification whereas $\FMeasure_{SP}$ = $0.62$ and $\FMeasure_{Odd}$ = $0.79$ suggest otherwise. Similar to the race attribute, $\FMeasure_{SP}$ = $0.62$ implies 
higher salary classification
bias in favour of Males. 
As for epistemic and aleatoric uncertainties, we observe gaps similar to the race attribute: There is significant bias in terms of $\FMeasure_{Epis}$ (against Males), despite the dataset containing more Male samples. Moreover, the model appears to have the same level of aleatoric certainty across gender ($\FMeasure_{Alea} \sim 1.00$).  

\paragraph{Social Impact:}
The point-based measures seem to indicate that the outcome is acceptably fair which is non-indicative of the underlying problem, i.e., the model is still unsure of its prediction of the majority class despite having more samples on them.
%
Hypothetically,
this could lead to prediction bias when encountering \textit{real-world} issues such as missing data \cite{goel2021importance} and distributional shifts \cite{chen2022fairness}.
%
Future experiments can be conducted to verify how such real-world challenges, e.g., missing data and distributional shifts, may impact the uncertainty-based fairness measures.
%
Our uncertainty-based fairness measures managed to 
highlight this discrepancy
across both race and gender which could encourage pre-emptive efforts to further investigate the underlying source of bias in the model before deploying them in real-world settings.

\subsubsection{D-Vlog Dataset} 
\label{sect:additional_DVlog}
\paragraph{D-Vlog Truncation Statistics}

Table \ref{tab:dvlog_duration_dist} shows that Female videos are truncated significantly, which leads to loss of information and an increase of uncertainty in predictions \cite{cheong2023towards}.
\begin{table} [hbt!]
\caption{Label, duration and sensitive attribute distributions of D-Vlog. Both average duration and average truncated amount are given in seconds. Absolute value of the entries with negative value in the last row shows the amount of zero padding whereas the positive values directly state the amount of truncation.
    \label{tab:dvlog_duration_dist}}
  \centering\small
  \begin{tabular}{c|cc|cc}
    \toprule
     & \multicolumn{2}{c|}{\textbf{Male}} & \multicolumn{2}{c}{\textbf{Female}} \\
    \midrule
    \textbf{} & $Y=0$ & $Y=1$ & $Y=0$ & $Y=1$ \\
    \midrule
    \# Samples   & 140 (0\%) & 182 (0\%)   & 2666 (0\%) & 373 (0\%) \\
    Avg. Duration(s)  & 483   & 583      & 587 & 667 \\
    Avg. Truncation(s)  & -158  & -13     & -9 & +71 \\
   \bottomrule
  \end{tabular}
\end{table}
\begin{table}[htb!]
  \caption{Experiment 2: The analysis with D-Vlog. Unfair values are \highlight{highlighted}. F/M: Females/Males. $G_0$: Females. \label{tab:dvlog_results}
    }
  \centering\small
  \begin{tabular}{ccccccc}
    \toprule
    & \multicolumn{2}{c}{{Multi-Modal}} & \multicolumn{2}{c}{{Audio Only}} & \multicolumn{2}{c}{{Visual Only}}\\
    \midrule
    \textbf{Measure}  & \textbf{F} & \textbf{M} & \textbf{F} & \textbf{M} & \textbf{F} & \textbf{M} \\
    \midrule
    $\uparrow$ Sample Size  & {639}  & 322 & {639}  & 322 & {639}  & 322 \\
    \hdashline[1pt/1pt]
    \multicolumn{7}{l}{\it\scriptsize Performance Measures}\\
    $\uparrow$ $\PMeasure_{Acc}$     & 0.59 & 0.73 & {0.63}   & 0.75    & {0.63} & 0.66\\
    $\uparrow$ $\PMeasure_{PPV}$     & 0.62 & 0.78 & 0.72   & {0.82}    & 0.65 & {0.69}\\
    $\uparrow$ $\PMeasure_{NPV}$     & 0.54 & 0.66 & {0.56}   & 0.66    & {0.61} & 0.57\\
    $\downarrow$ $\PMeasure_{FPR}$   & 0.57& 0.38 & 0.29   & 0.28    & {0.54} & 0.61\\
    $\downarrow$ $\PMeasure_{FNR}$   & 0.28 & 0.20 & {0.43}   & 0.22    & {0.23} & 0.18\\
    $\downarrow$ $\Unce$    & 0.006 & 0.006 & 0.022 & 0.016  & 0.034 & {0.035}\\
    $\downarrow$ $\Unca$   & 0.45 & 0.45 & 0.28 & 0.22  & 0.10 & {0.09}\\
    $\downarrow$ $\Unc_p$  & 0.46 & 0.46 & 0.31 & 0.24  & 0.14 & {0.13}\\
    \hdashline[1pt/1pt]
    \multicolumn{7}{l}{\it\scriptsize Point-based Fairness Measures}\\
    $\FMeasure_{SP}$  & \multicolumn{2}{c}{{1.01}} & \multicolumn{2}{c}{{\highlight{0.75}}} & \multicolumn{2}{c}{0.91}\\
    $\FMeasure_{Opp}$  & \multicolumn{2}{c}{0.89} & \multicolumn{2}{c}{\highlight{0.73}} & \multicolumn{2}{c}{0.94} \\
    $\FMeasure_{Odd}$  & \multicolumn{2}{c}{{\highlight{1.68}}} & \multicolumn{2}{c}{\highlight{1.40}} & \multicolumn{2}{c}{0.94}\\
    $\FMeasure_{EAcc}$  & \multicolumn{2}{c}{0.81} & \multicolumn{2}{c}{0.84} & \multicolumn{2}{c}{0.96}\\
    \hdashline[1pt/1pt]
    \multicolumn{7}{l}{\it\scriptsize Uncertainty-based Fairness Measures (Ours)}\\
    {$\FMeasure_{Epis}$} & \multicolumn{2}{c}{{1.00}} & \multicolumn{2}{c}{\highlight{1.38}} & \multicolumn{2}{c}{0.96} \\
    {$\FMeasure_{Alea}$} & \multicolumn{2}{c}{{1.00}} & \multicolumn{2}{c}{\highlight{1.32}} & \multicolumn{2}{c}{1.11}\\
    {$\FMeasure_{Pred}$} & \multicolumn{2}{c}{{1.00}} & \multicolumn{2}{c}{\highlight{1.32}} & \multicolumn{2}{c}{1.07}\\
    \bottomrule
  \end{tabular}
    
    
\end{table}
As the dataset owners have explored both multi-modal and uni-modal architectures, we analyze D-Vlog both in a multi-modal and in a uni-modal manner. Table \ref{tab:dvlog_results} provides the experimental results. Both point-based and uncertainty-based fairness measures deem the classifier to be fair (except for $\FMeasure_{Odd}$). Uncertainty-based fairness results are especially surprising since the Female group size is twice the size of the Male group. However, we observe that the classifier has high aleatoric uncertainty for both groups.

The results per modality suggest that the audio modality has strong bias against Females since the performance measures are generally lower for Females. This, however, is not coherently captured by point-based measures whereas our uncertainty-based measures consistently highlight the bias. The cause of this bias appears to be the truncation of the videos by the dataset owners: Recordings of Females are significantly longer and therefore, truncated more. 
This naturally results in more reduction in information useful for the classification task for females, thus increasing the uncertainty for females. However, this effect is not observed across the visual modality 
as the classifier performs poorly across both males and females.


\subsection{Experiment 3: Individual Fairness}

We now analyze individual fairness with point-based and uncertainty measures for COMPAS. 
The results in Fig. \ref{fig:indv_fairness_scores} suggest that $\FMeasure^{indv}_{\hat{y}}$ values differ across different groups of race and gender as well as outcomes. This is also evident with the uncertainty-based individual fairness measures ($\FMeasure^{indv}_{\Unc}$). However, although ``W(-)" and ``B(-)" have similar point-based consistencies, they are very different across both aleatoric and epistemic consistencies. Aleatoric consistencies align with $\FMeasure^{indv}_{\hat{y}}$ that the classifier is having difficulty with ``B(+)" samples. $\FMeasure^{indv}_{\Unce}$ values highlight that ``M(+)" and ``B(+)" might especially benefit from additional data. 

\subsubsection{Experiment 3: Individual Fairness Analysis on Adult}
\label{sect:additional_indv_fairness}

With reference to Figure \ref{fig:indv_fairness_scores_adult}, we see that $\FMeasure^{indv}_{\hat{y}}$ for both race and gender are largely similar.
This is also true for the the Uncertainty-based individual fairness measures $\FMeasure^{indv}_{\Unca}$ and $\FMeasure^{indv}_{\Unce}$.
A noteworthy point is that all measures indicate an interesting insight about the positive classes, (``B+", ``W+", ``F+" and ``M+").
All of them point towards a perfect consistency score of $\approx1$.
We hypothesize that this might be due to the severe class imbalance within the Adult dataset where there is a very small subgroup that belongs to the positive class $Y=1$ thus causing the classifier to memorize and be highly confident about the $\hat{Y}=1$ predictions. 
This is also supported by the small (~0.07) FPR reported in Table 4 (the main manuscript) for all groups. We see slightly lower consistency for negative classes in point predictions and aleatoric uncertainty. The high FNR rate for both groups (Table 4 in the main text) suggests that there are more errors with $\hat{Y}=0$ predictions, causing higher inconsistencies for those predictions. Lower consistencies for ``F-" and ``B-" for epistemic uncertainty suggest more data can be helpful for Female and Black groups, which is supported by the dataset distribution highlighted in Table \ref{suptab:compas_fairness_results}. 

\paragraph{Social Impact:} Considering uncertainty in individual fairness can be crucial in many applications. For instance, cancer-free prognosis of a patient should take into account uncertainty-based consistency for similar individuals.

\begin{figure}[hbt!]
  \centerline{
  {
    \includegraphics[width=0.38\columnwidth]{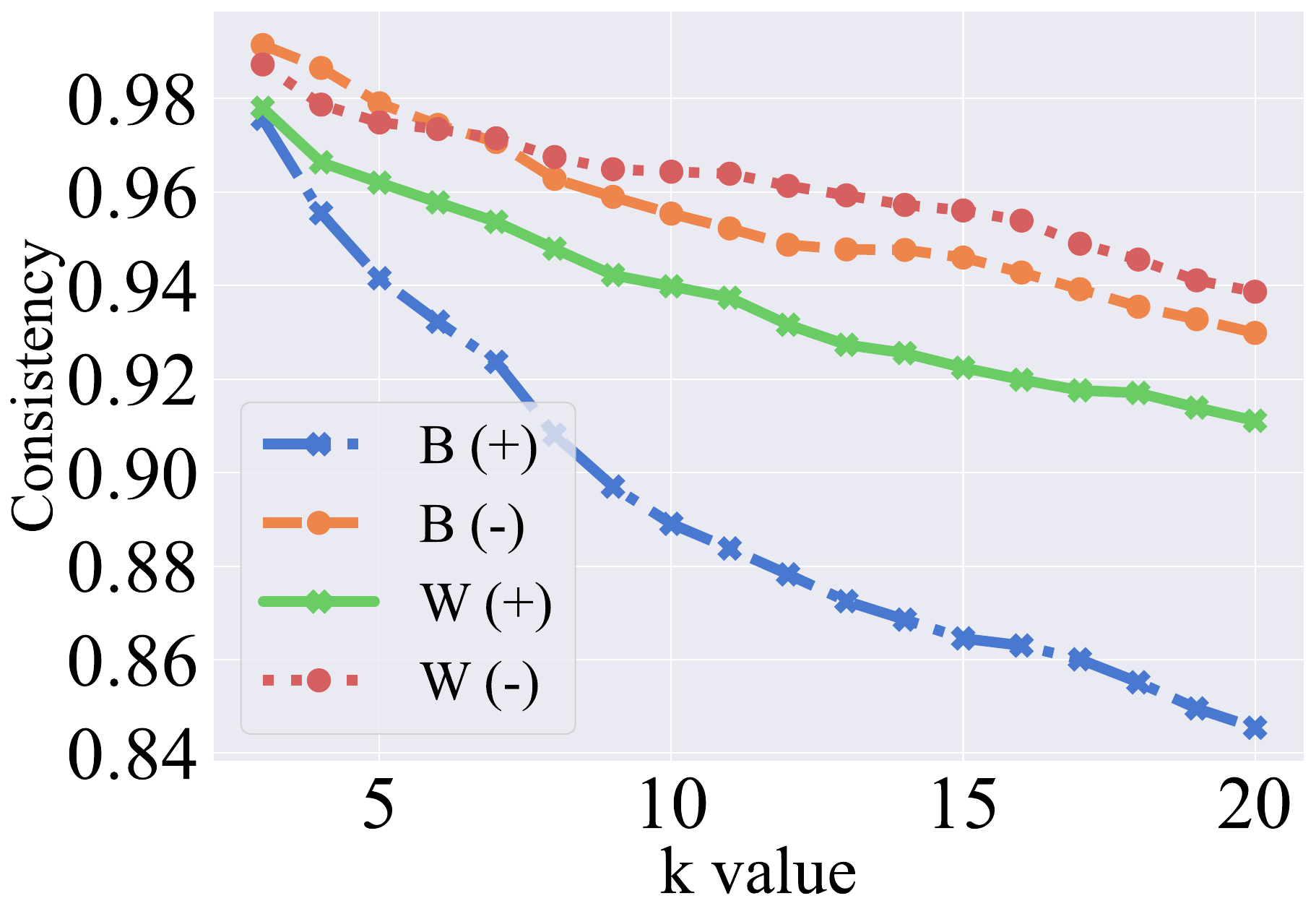}
    \label{fig:compas_race_point}
 }
 {
    \includegraphics[width=0.38\columnwidth]{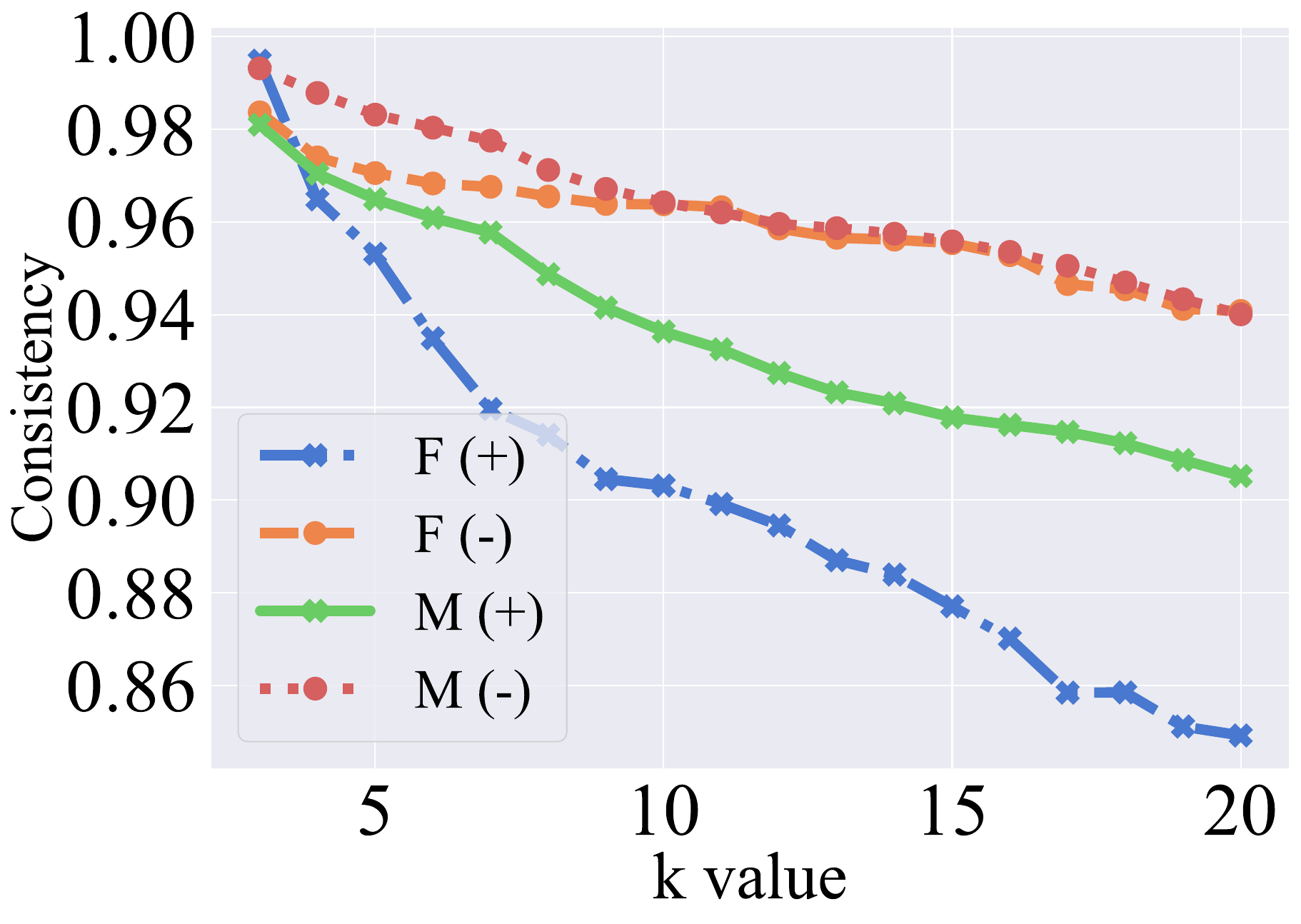}
    \label{fig:compas_gender_point}
  }
  }
  \centerline{
  {
    \includegraphics[width=0.38\columnwidth]{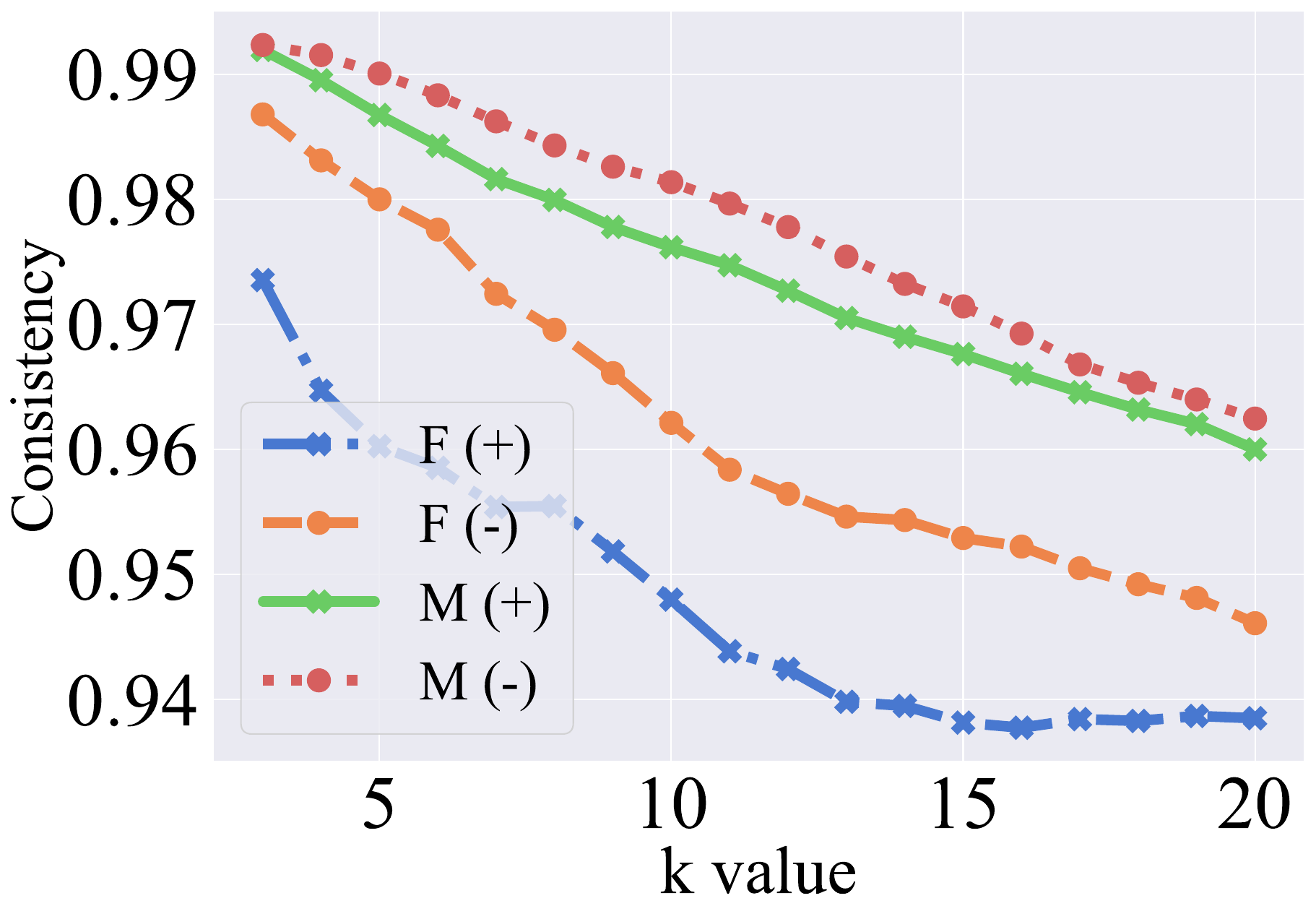}
    \label{fig:compas_gender_aleatoric_indie}
  }
  {
    \includegraphics[width=0.38\columnwidth]{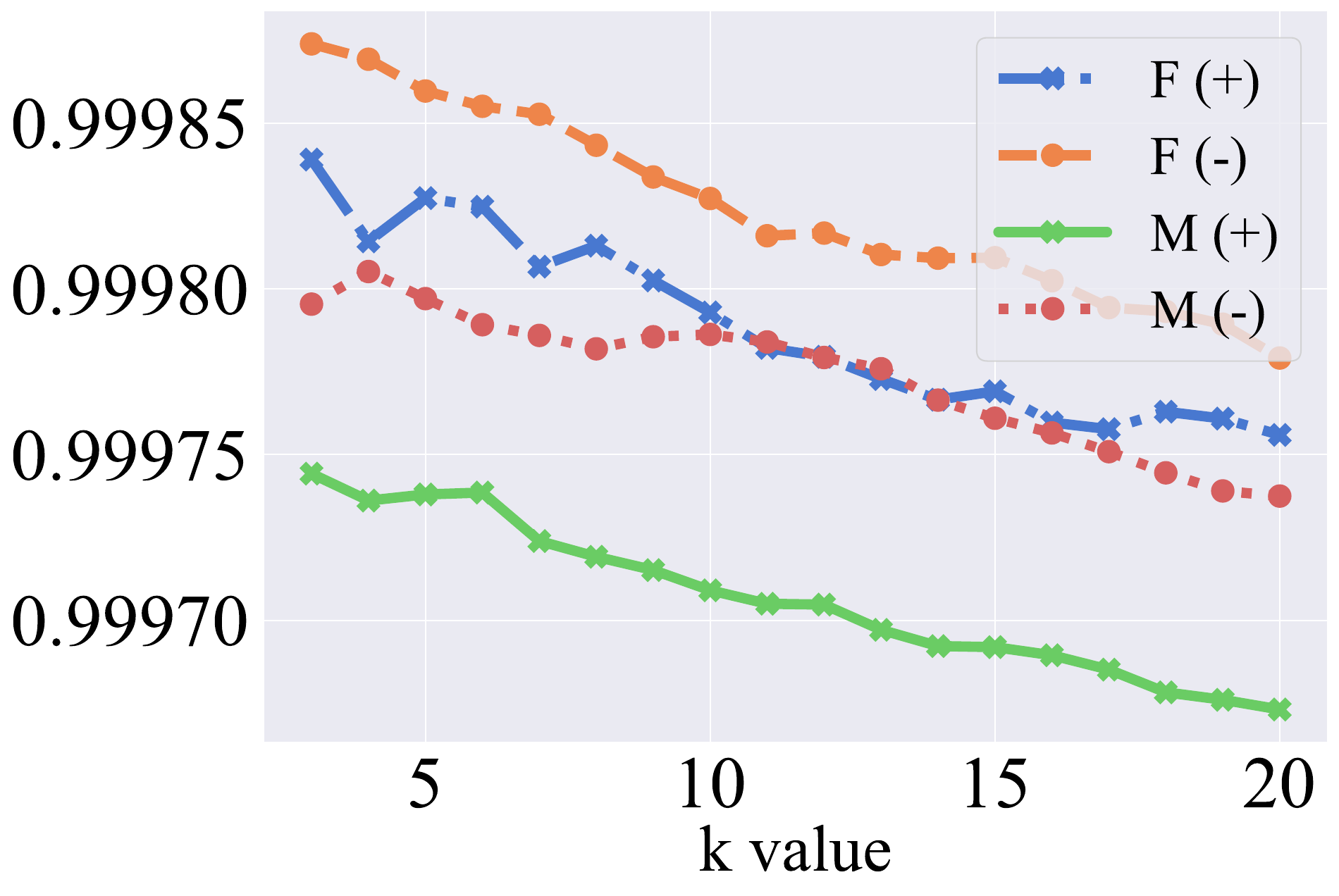}
    \label{fig:compas_gender_epistemic_inde}
  }
  }
  \centerline{
  {
    \includegraphics[width=0.38\columnwidth]{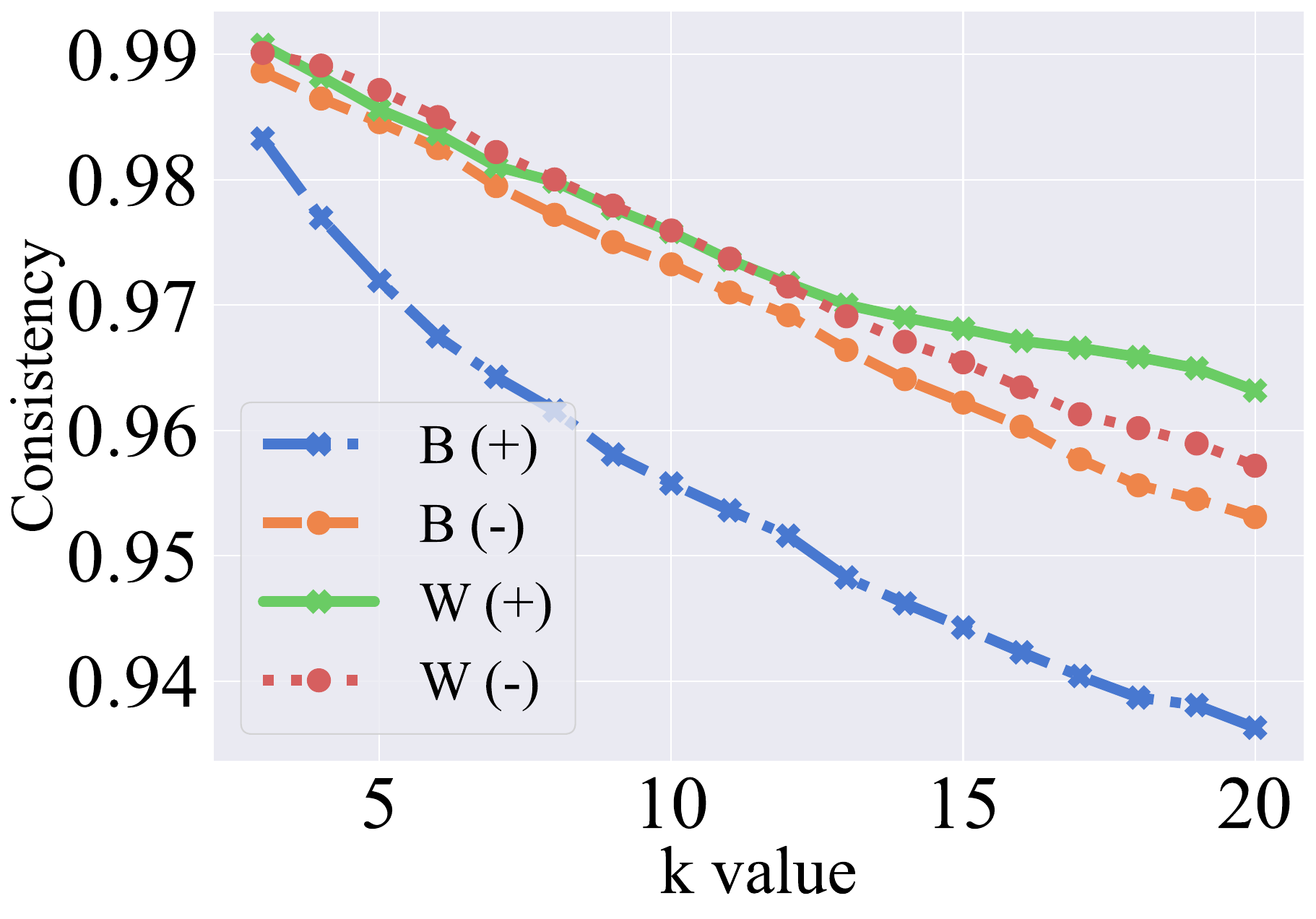}
    \label{fig:compas_race_aleatoric_indie}
 }
 {
    \includegraphics[width=0.38\columnwidth]{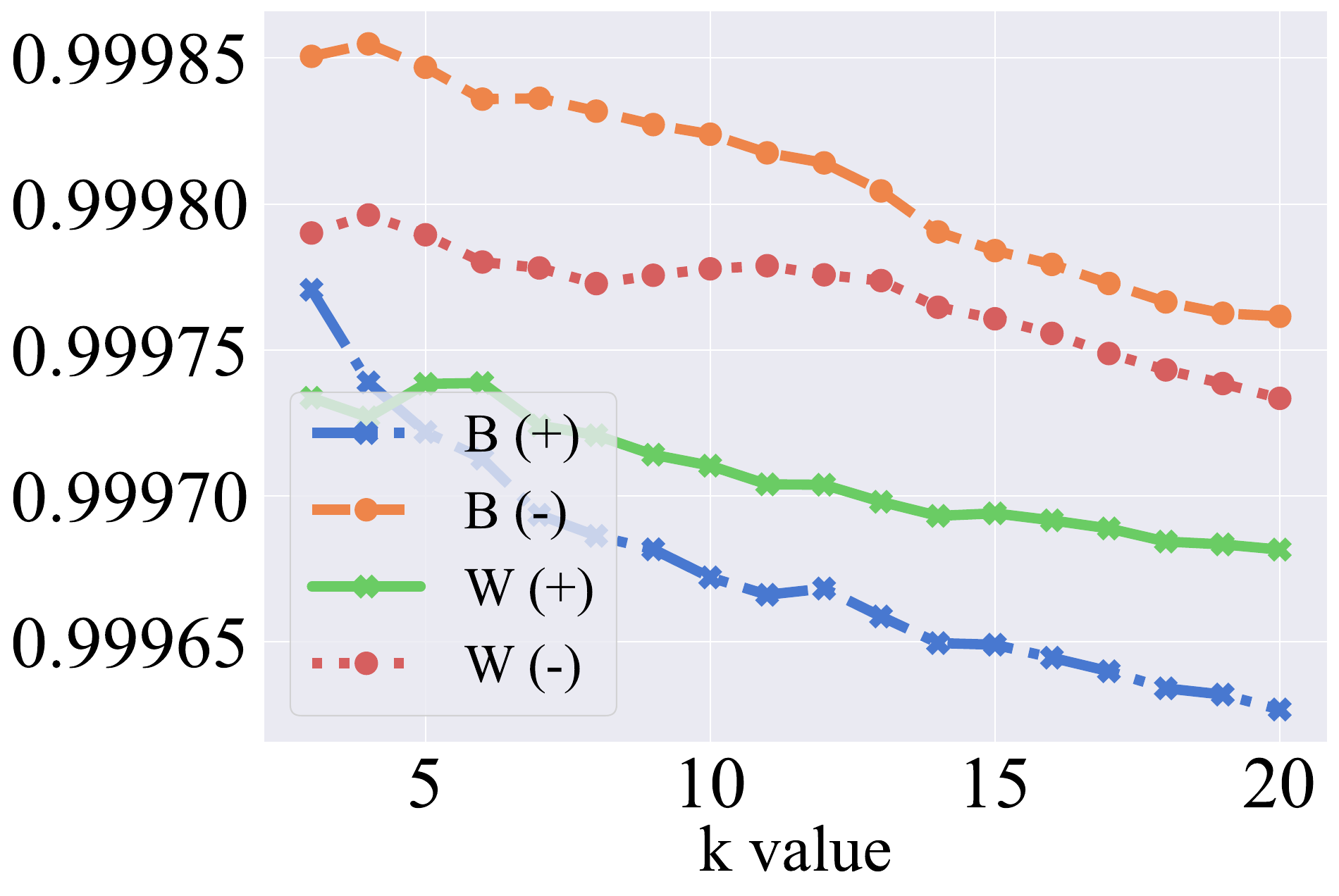}
    \label{fig:compas_race_epistemic_indie}
  }
  }
\caption{Experiment 3: Point-based \textbf{(a,b)} and uncertainty-based individual fairness \textbf{(c-f)} scores for COMPAS.} 
\label{fig:indv_fairness_scores} 
\end{figure}

\subsection{Experiment 4: Ablation Analysis}
\label{sect:additional_Ablation}
We analyze the effect of model capacity on performance and uncertainty estimations. The results in Fig. \ref{fig:compas_ablation} show that the uncertainty estimations are affected by the change in the number of neurons per layer. However, the relative ordering between the different demographic groups do not appear to be affected. Since accuracy appears to saturate after 100 neurons and to lower the computational cost, we have chosen the hidden layer sizes as 100 in all experiments.
Adding more layers led to significant over-fitting problems for SD1, SD2, SD3, COMPAS, and Adult datasets. Therefore, we performed the rest of the experiments with a single hidden-layer for COMPAS and no hidden-layer for SD1, SD2, SD3 and Adult.

\begin{figure}[hbt!]
  \centerline{
  {
    \includegraphics[width=0.38\columnwidth]{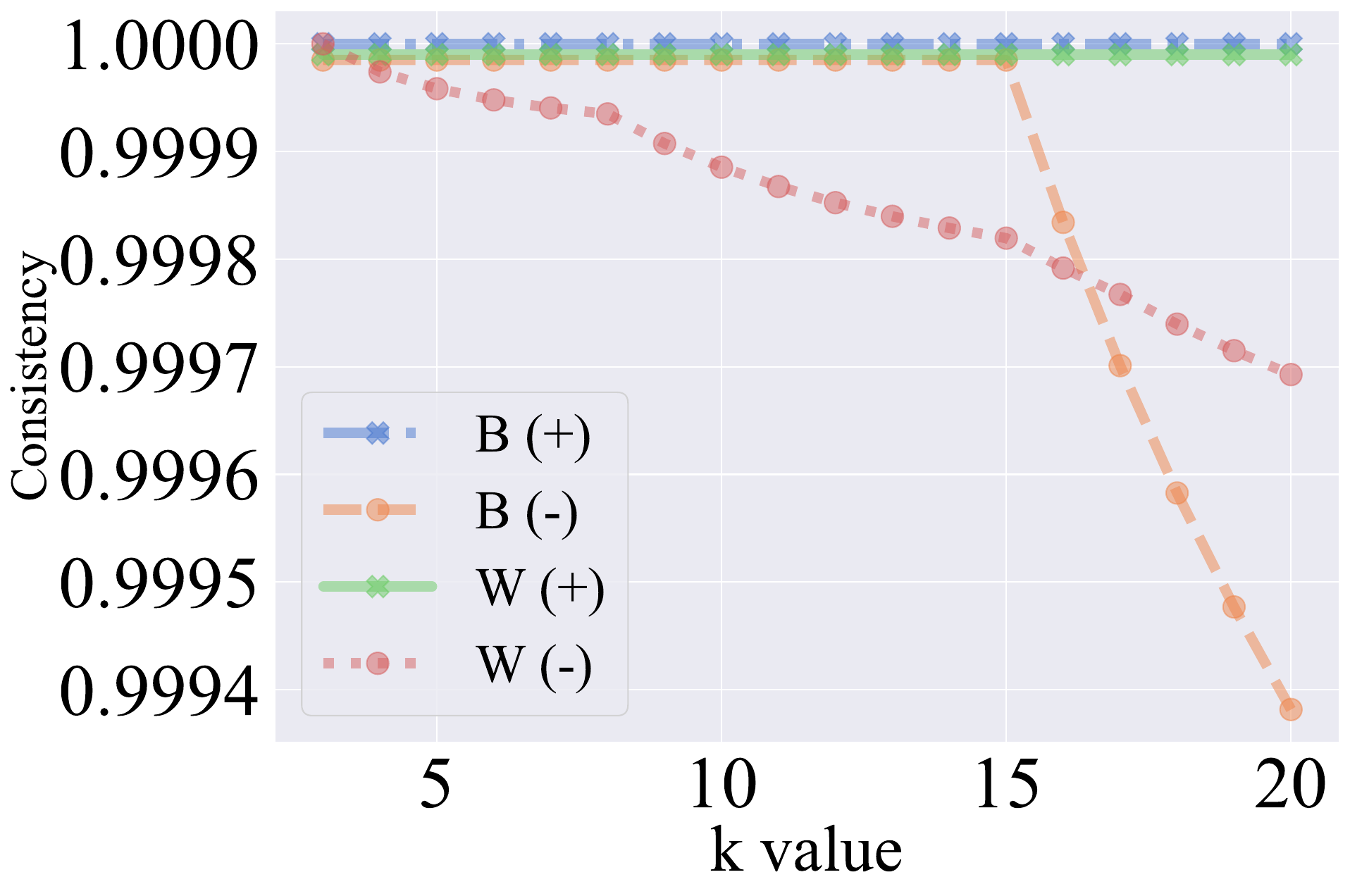}
    \label{fig:adult_race_point}
 }
 {
    \includegraphics[width=0.38\columnwidth]{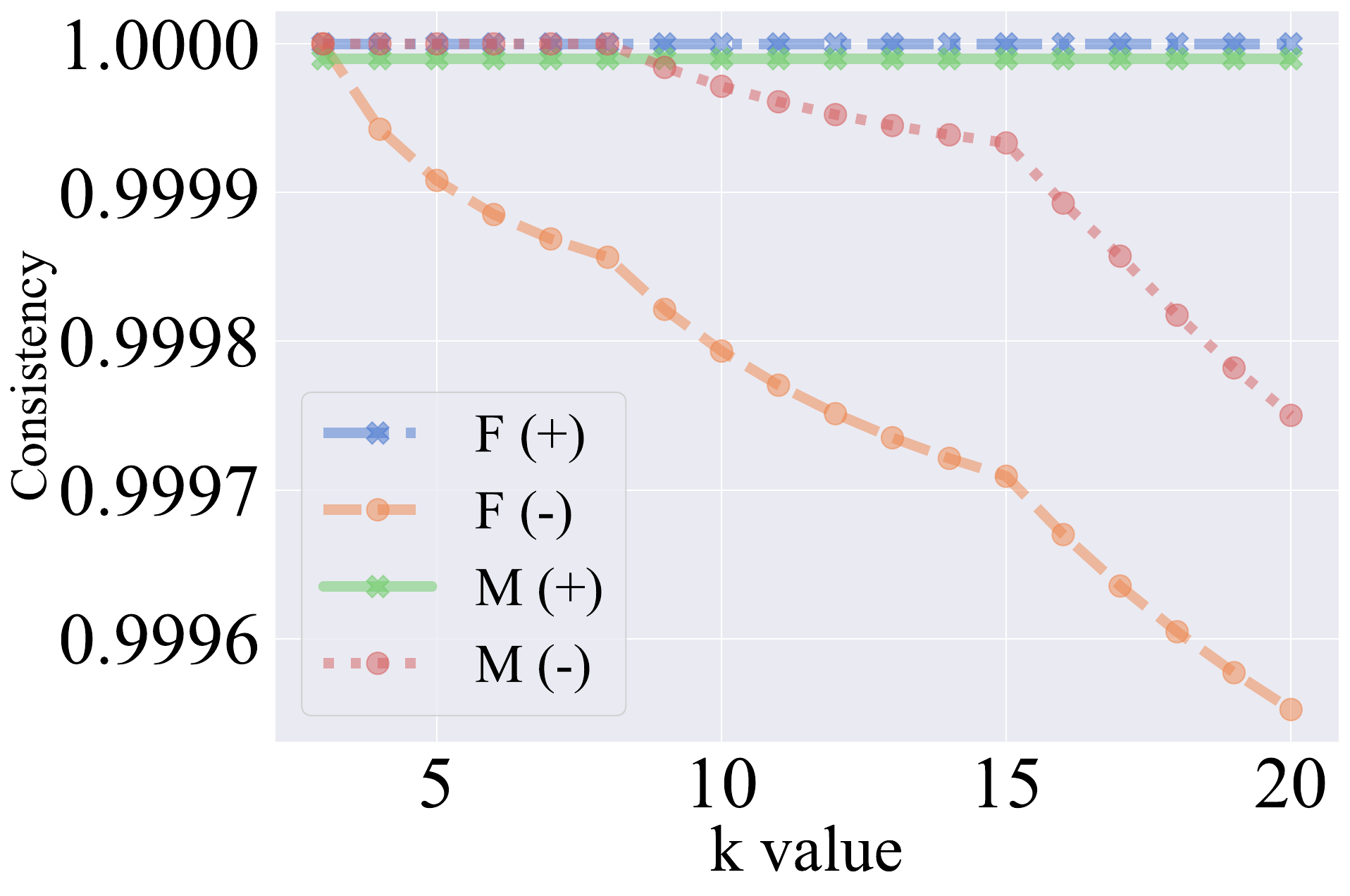}
    \label{fig:adult_gender_point}
  }
  }
  \centerline{
  {
    \includegraphics[width=0.38\columnwidth]{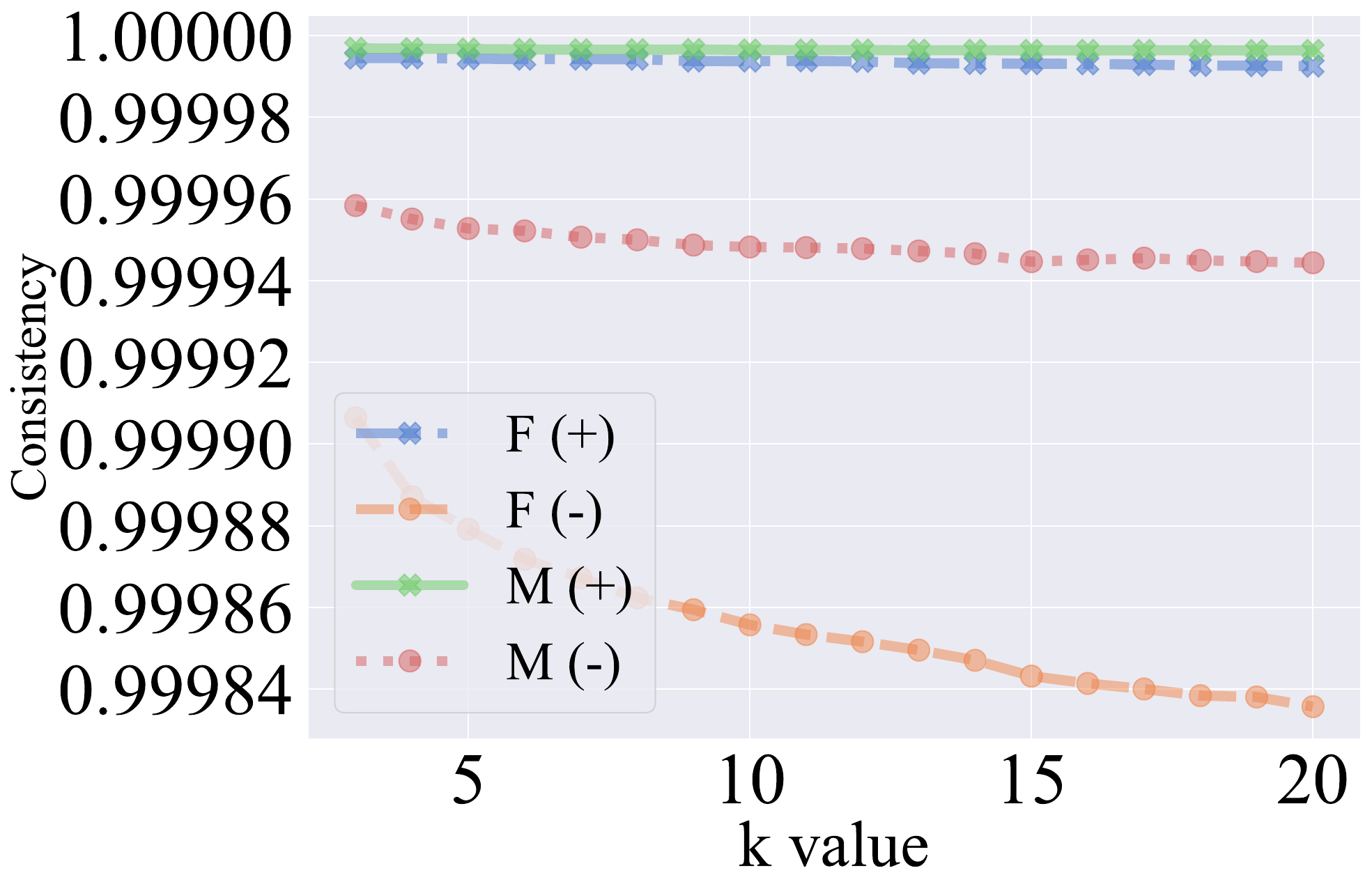}
    \label{fig:adult_gender_aleatoric_indie}
  }
  {
    \includegraphics[width=0.38\columnwidth]{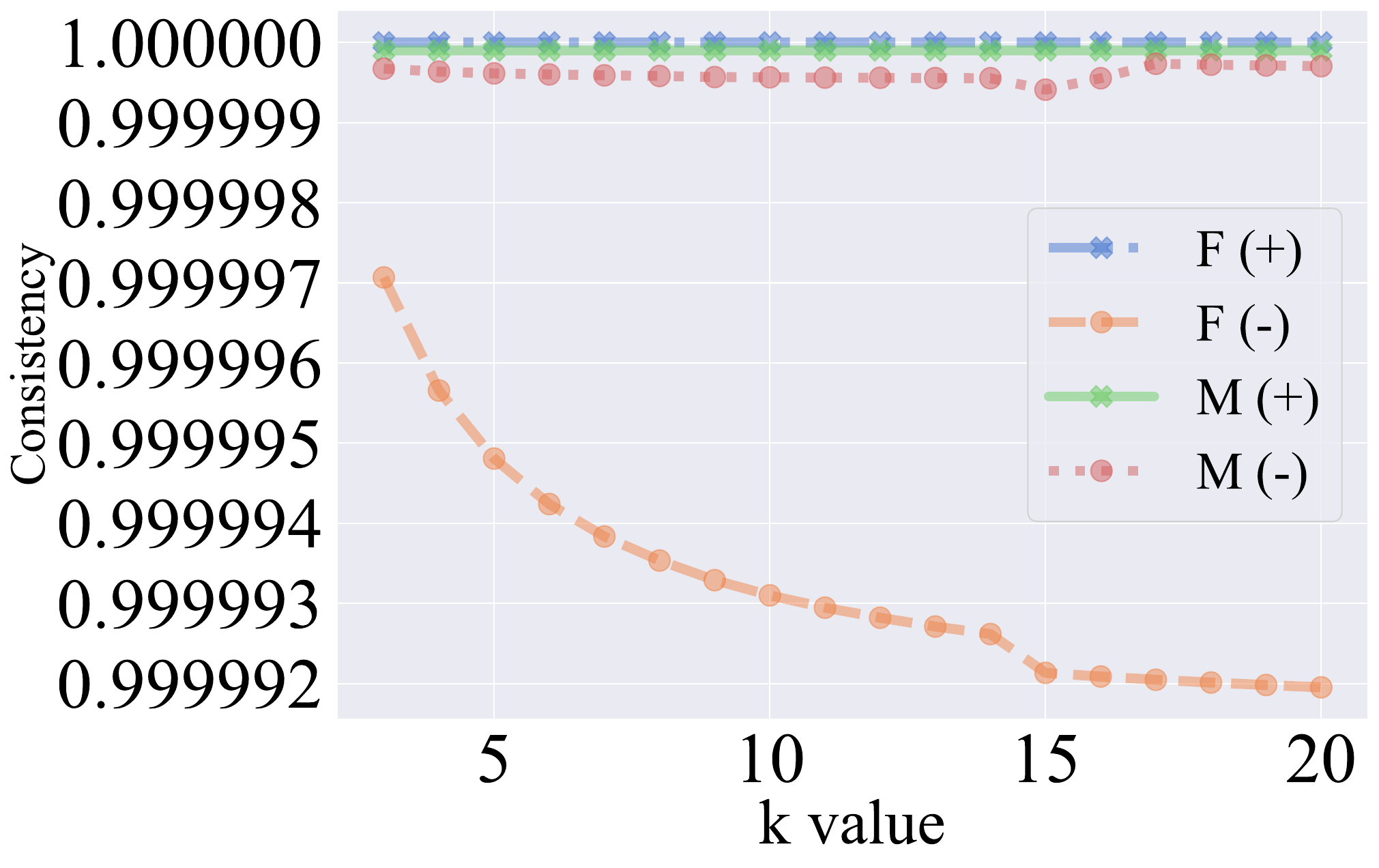}
    \label{fig:adult_gender_epistemic_inde}
  }
  }
  \centerline{
  {
    \includegraphics[width=0.38\columnwidth]{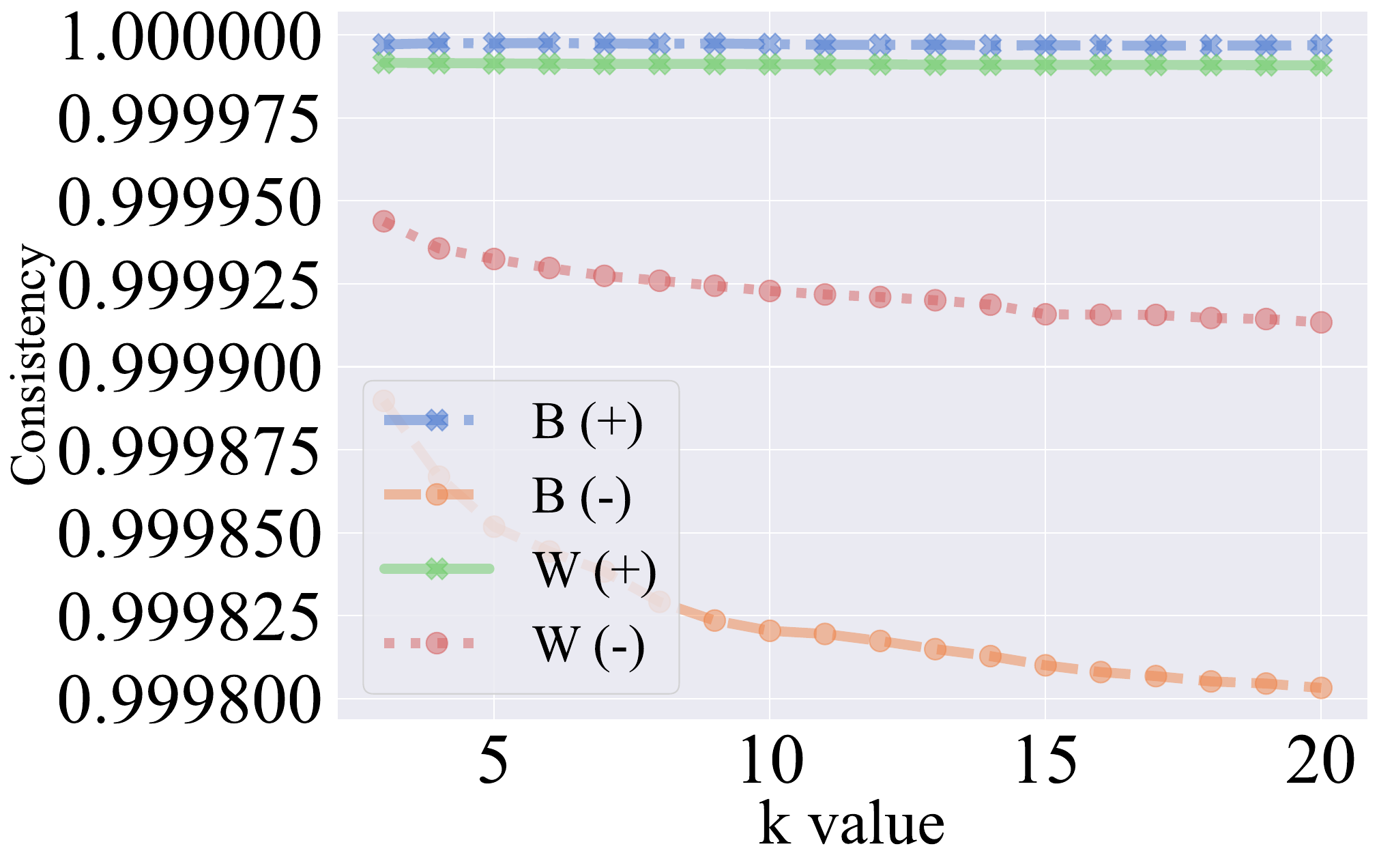}
    \label{fig:adult_race_aleatoric_indie}
 }
 {
    \includegraphics[width=0.38\columnwidth]{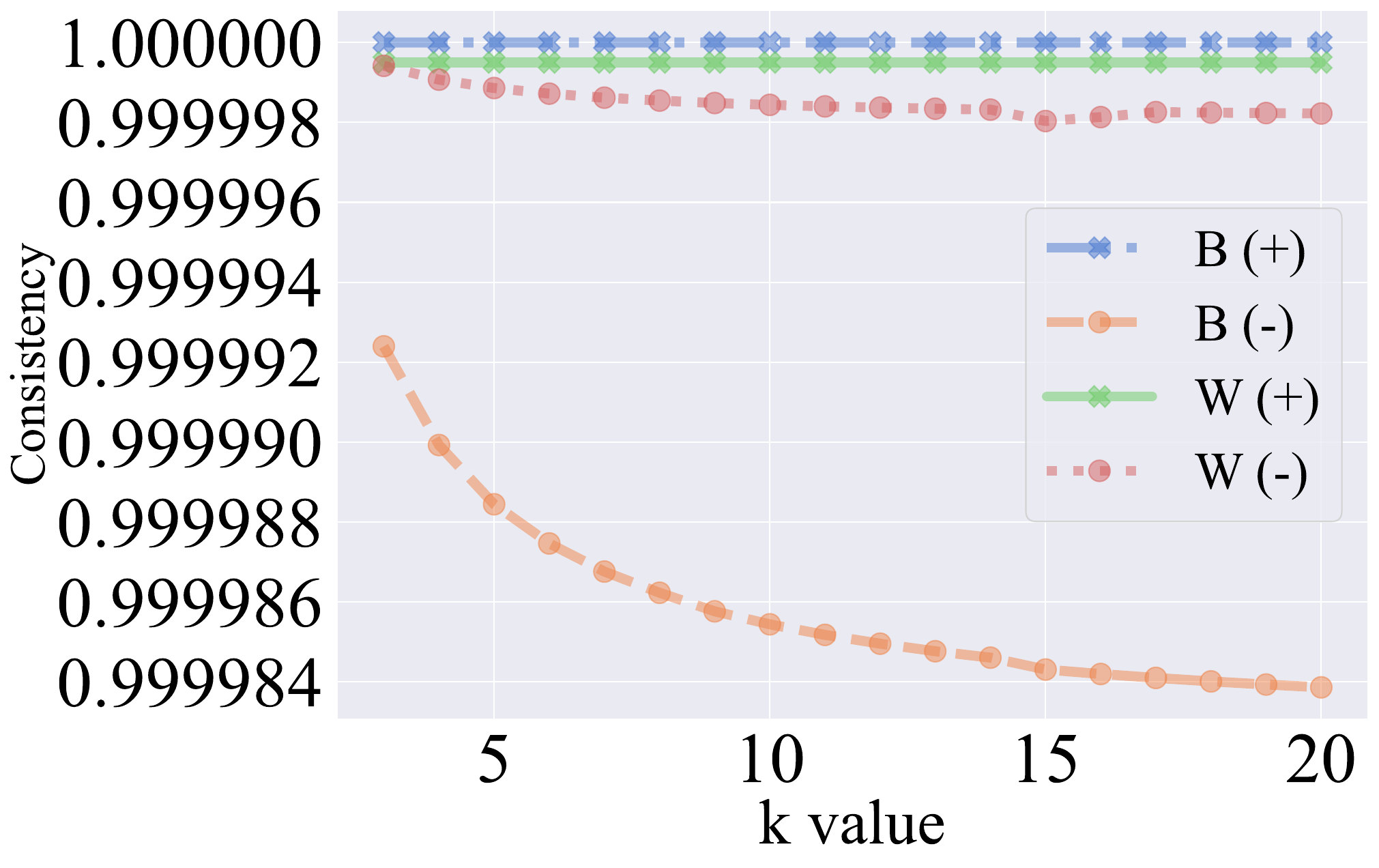}
    \label{fig:adult_race_epistemic_indie}
  }
  }
  
\caption{Experiment 3: Point-based \textbf{(a,b)} and uncertainty-based individual fairness \textbf{(c-f)} scores for Adult. 
}
\label{fig:indv_fairness_scores_adult}
\end{figure}

\begin{figure}[hbt!]
\centering
  \subfigure[Accuracy]
  {
    \includegraphics[width=0.31\columnwidth]{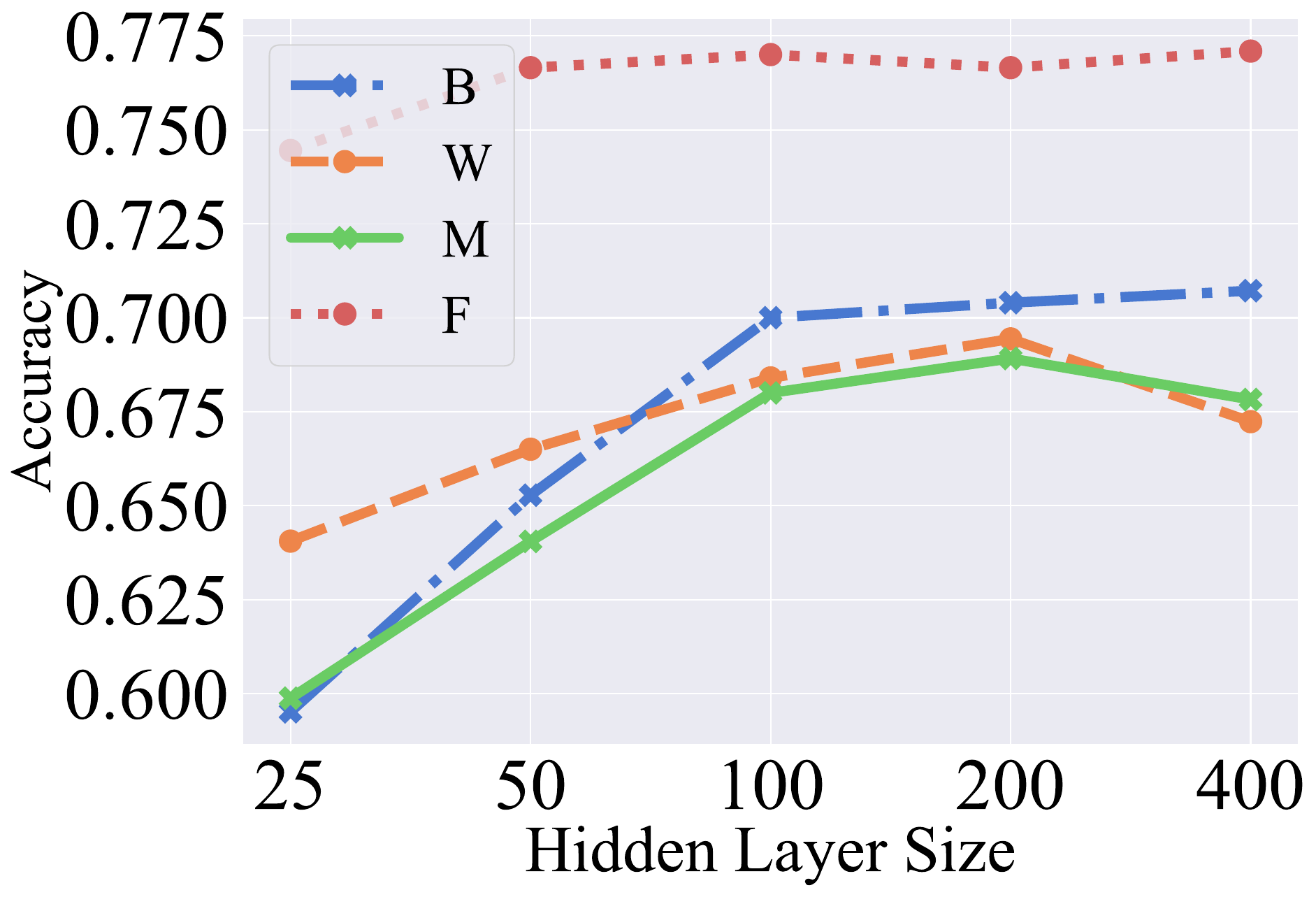}
    \label{fig:compas_ablation_accuracy}
  }
  \subfigure[Aleatoric Uncertainty]
  {
    \includegraphics[width=0.31\columnwidth]{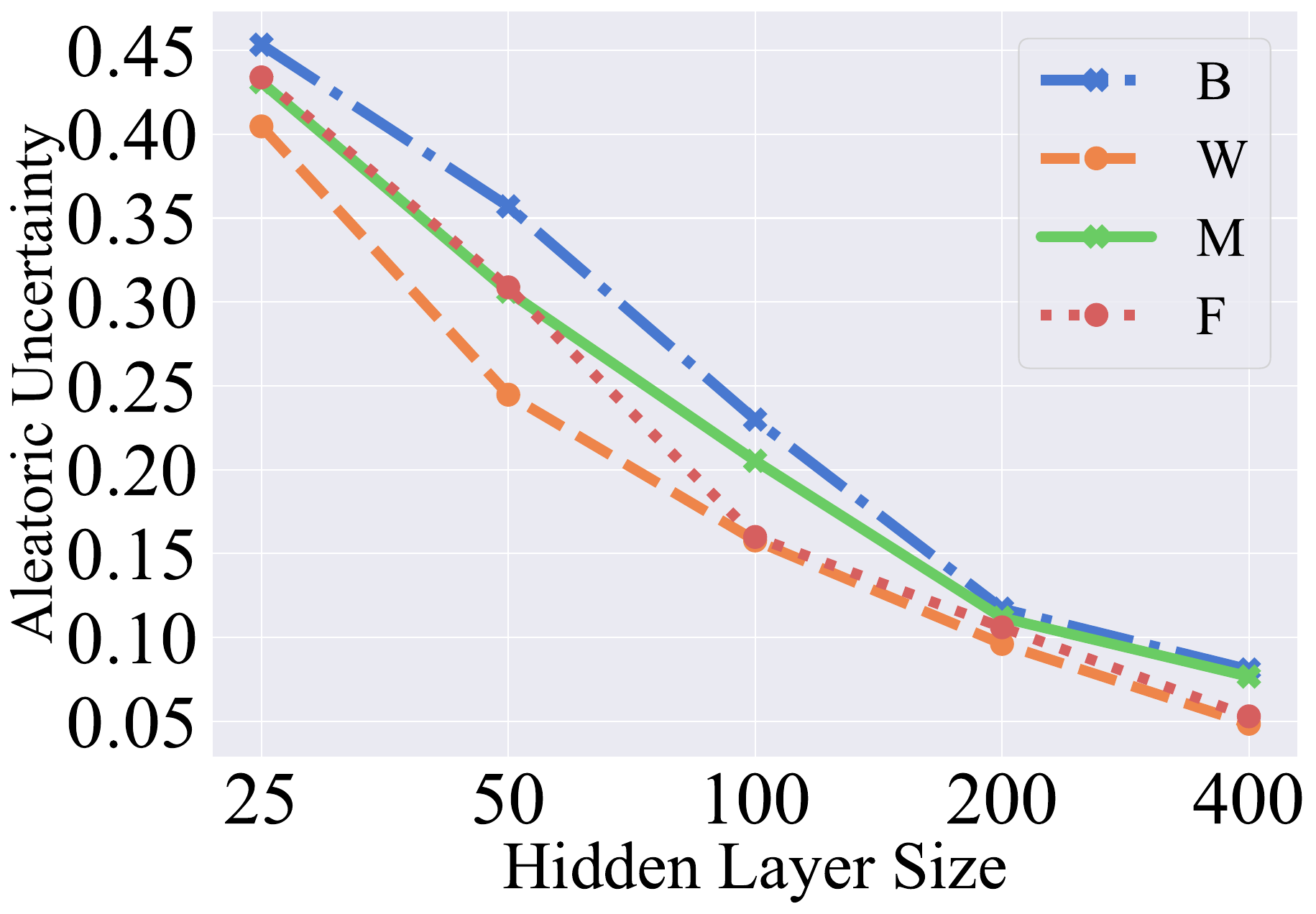}
    \label{fig:compas_ablation_aleatoric}
  }
  \subfigure[Epistemic Uncertainty]
  {
    \includegraphics[width=0.31\columnwidth]{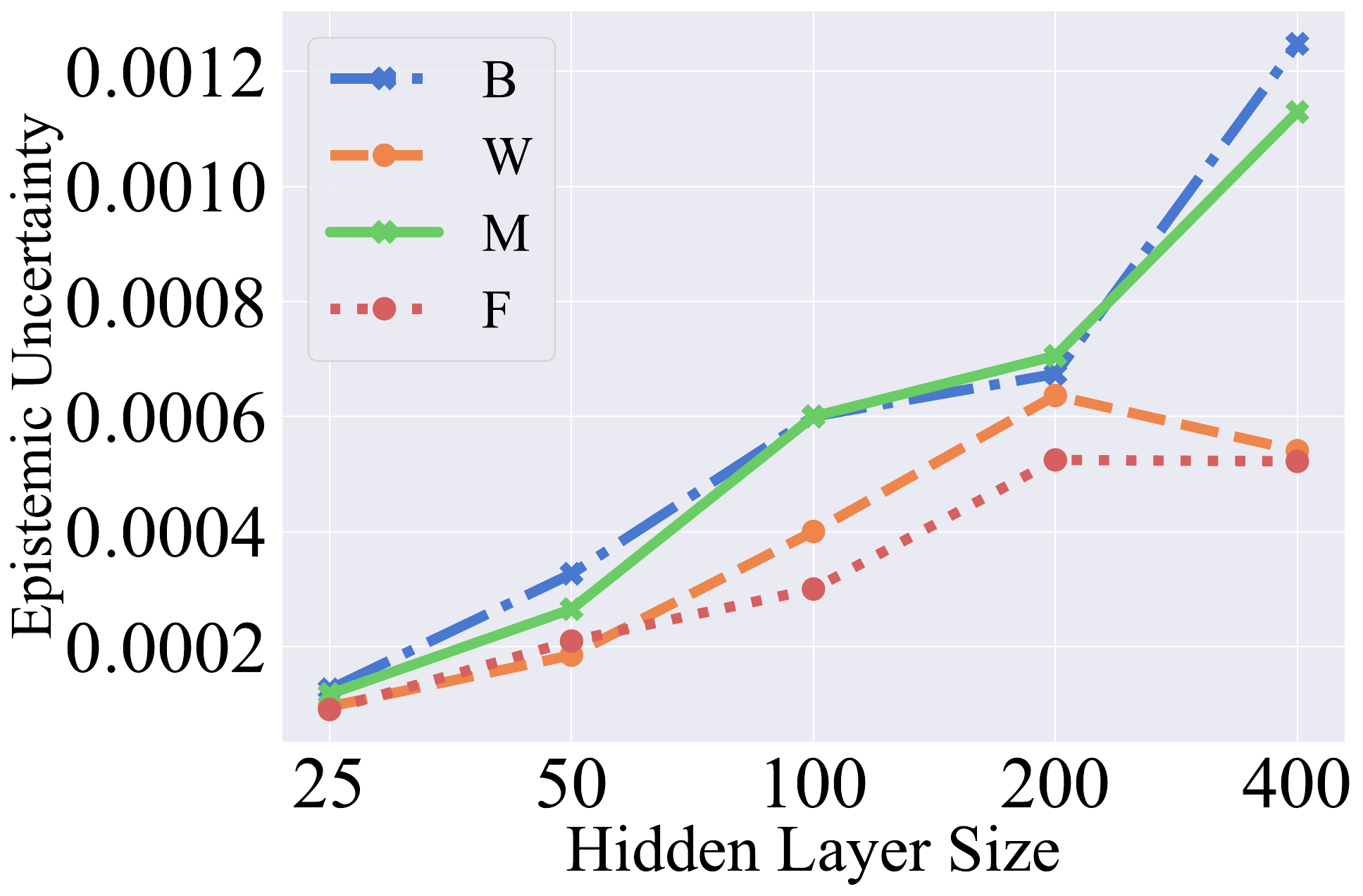}
    \label{fig:compas_ablation_epistemic}
 }
\caption{Experiment 4: Ablation analysis on the effect of model capacity on COMPAS. \textbf{(a)} Accuracy. \textbf{(b)} Aleatoric uncertainty. \textbf{(c)} epistemic uncertainty. B/W: Black/White. F/M: Female/Male. As model performance starts saturating at 100 neurons, we have used 100 neurons in BNNs.
}
\label{fig:compas_ablation}
\end{figure}

\section{Discussion} 
\label{sect:conclusion}

In this paper, we have argued that existing point-based fairness measures may not be reliable as they depend on point predictions of the ML models and ignore their uncertainties. 
To address this limitation, we introduce the use of different types of uncertainty as fairness measures. We prove that the proposed fairness measures are independent of point-based fairness measures and empirically show that uncertainty-based fairness measures provide more insights about the presence and the source of bias in predictions. 

\subsection{Main Insights}
In the following, we summarize the main insights:

\paragraph{Insights through the Epistemic Fairness Measure ($\FMeasure_{Epis}$)} Measuring the fairness gap in terms of epistemic uncertainty, by definition, highlights the lack of data for one group. What is beneficial is that this is not affected by the mere number of samples, which can be misleading. For example, in COMPAS, Black and Male groups have significantly more samples. However, both groups still witness higher $\Unc_e$ values. This suggests that the dataset may contain data-level bias. The dataset distribution confirms that there is a class-imbalance problem for these groups, which can be remedied with more data.

\paragraph{Insights through the Aleatoric Fairness Measure ($\FMeasure_{Alea}$)} Aleatoric uncertainty reflects the hardness of a problem owing to label or data noise (e.g., occlusion) \cite{kendall2017uncertainties}. The use of this informative measure has shown that the classification task is harder for some groups. For example, in D-Vlog, we see that truncating videos has increased aleatoric uncertainty for females.

Another prominent example is that of COMPAS. For instance, despite being a frequently used benchmark for fairness evaluation, an oft-cited key limitation of COMPAS is that errors in typography is a major flaw in this dataset \cite{rudin2019stop}.
Uncertainty can, by definition, capture some of the typography and data issues, which would be missed by point-based measures. 
As evidenced in Table \ref{suptab:compas_fairness_results}, 
uncertainty-based fairness measures provided some
insights about the roots of bias and can be used in conjunction with point-based measures.

\subsection{Social Impact}
There is a rapid increase in bias mitigation methods for the past years \cite{hort2022bia}. 
However, it is unclear which source of bias each method is intended to address. In fact, recent work has demonstrated that if bias is due to missing values, existing bias mitigation methods often reduce (point-based) performance disparities at the cost of accuracy
\cite{wang2023aleatoric}.
Our contribution lies in leveraging existing uncertainty measures to quantify an alternative aspect of fairness. That said, probing a model by adding noise or perturbations to its inputs is useful in analyzing model robustness or increasing model robustness if noise or perturbations are added during training. Epistemic and aleatoric uncertainties, on the other hand, pertain to how well the model captures the lack of data and the absence of noise (ambiguity) respectively. Given a dataset and a model, both types of uncertainties are supposed to be irreducible.
In such an instance, using point-based measures will likely be sub-optimal. Our proposed uncertainty-based measure highlights this underlying problem and cautions against foisting a ``fair'' outcome using point-based fairness measures. 

Moreover, many of the existing bias mitigation solutions rest on strict machine learning assumptions such as 
having access to clean or noise-free labels 
and requiring the model to be deployed in a fair environment that does not deviate from the training setting \cite{kang2022certifying}.
This is optimistic at best and harmful at worst.
This incongruence between theoretical formulation and real-world settings is one of the handicaps that the machiner learning fairness research community needs to overcome.
Our work also highlights the need to develop methods which are able to address epistemic and aleaoteric sources of discrimination. 
We hope that the proposed uncertainty-based fairness measures present a step towards that direction.

\subsection{Limitations} 
Despite its merits, uncertainty-based fairness measures require working with models which provide or can be modified to provide uncertainties. 
Moreover, quantifying uncertainty is an active research area, and in this work we have not been able to undertake a thorough evaluation of different uncertainty quantification methods. 
The above provides opportunities for future work. 
%
A key point to note is that the uncertainty-fairness measures in our paper are differentiable and can be converted to a loss function. However, forcing epistemic and aleatoric uncertainties to be similar across groups or individuals will not necessarily change the “real uncertainties” as these measures simply reflect issues inherent in the data and noise (or ambiguity).

Although prediction uncertainty can be helpful in analyzing fairness, this approach has certain limitations which we view as opportunities for future work. For example, uncertainty estimation requires either using models that directly provide multiple predictions (e.g., BNNs, Deep Ensembles) or modifying models (and their training procedure) to do so (e.g., Monte Carlo Dropout \cite{gal2016dropout}). This hinders the use of state-of-the-art architectures (or their trained versions) in fairness analysis. Moreover, there is also the overhead involved with obtaining multiple predictions to quantify uncertainty. This can be alleviated with one-pass uncertainty estimation approaches, though they tend to be less reliable than the approaches considered in this paper {\cite{abdar2021review}}.

Quantifying uncertainty in a reliable manner is a challenging and an active research topic \cite{mukhoti2023ddu,liu2020sngp,amersfoort2020duq}. Although we have obtained similar outcomes with two different methods (BNNs and Deep Ensembles), we have encountered difficulties with the ranges of estimated uncertainties. It would be beneficial to perform our analyses with newer approaches. Another promising research direction is to consider alternative metrics for measuring the dispersion of uncertainty values in a group as taking the average across a group can miss important characteristics of the distribution.
Despite the aforementioned limitations, we sincerely hope that our work can provide a stepping stone towards investigating and addressing these challenges.


\section*{Acknowledgments}

Both first authors, Selim Kuzucu and Jiaee Cheong, contributed equally to this work. 
This work was undertaken while Jiaee Cheong was a visiting PhD student at the Middle East Technical University (METU).
\noindent\textbf{Open access:} 
The authors have applied a Creative Commons Attribution (CC BY) licence to any Author Accepted Manuscript version arising.
\noindent\textbf{Data access:} 
This study involved secondary analyses of existing datasets. All datasets are described and cited accordingly. 
\noindent\textbf{Funding:} 
J. Cheong is supported by the Alan Turing Institute Doctoral Studentship and the Leverhulme Trust, and further acknowledges resource support from METU during her visiting studentship. H. Gunes is supported by the EPSRC/UKRI project ARoEQ under grant ref. EP/R030782/1. We gratefully acknowledge the computational resources provided by METU
Center for Robotics and Artificial Intelligence (METU-ROMER) and METU
Image Processing Laboratory.


\appendix


\vskip 0.2in
\bibliographystyle{theapa}

\end{document}